\documentclass{article} 

\def\X{X} 

\def\ds{\displaystyle}

\usepackage{amsthm,amssymb,amsmath,wasysym}
\usepackage{amsmath, graphicx,epstopdf,enumerate,color,wasysym,subfigure}

\DeclareMathOperator{\R}{\mathbb{R}}

\DeclareMathOperator{\E}{\mathbb{E}}

\DeclareMathOperator{\F}{F}

\newtheorem{theorem}{Theorem}
\newtheorem{lemma}{Lemma}
\newtheorem{corollary}{Corollary}

\newtheorem{rem}{Remark}
\newtheorem{ex}{Example}
\newtheorem{ass}{Assumption}

\providecommand{\nor}[1]{\left\lVert {#1} \right\rVert}

\providecommand{\scal}[2]{\left\langle{#1},{#2}\right\rangle}

\newcommand{\RR}{\mathbb R}
\renewcommand{\H}{\mathcal{X}}
\newcommand{\PP}{\mathbb{P}}

\newcommand{\hh}{\mathcal H}
\newcommand{\EE}{\mathcal E}
\newcommand{\MM}{\mathcal M}
\newcommand{\CC}{\mathcal C}

\renewcommand{\SS}{\mathcal S}
\newcommand{\Sn}{S_n}
\newcommand{\Sk}{S_k}
\newcommand{\Snk}{S_{n,k}}
\newcommand{\Fn}{F_{n}}
\newcommand{\Fk}{F_{k}}
\newcommand{\Fnk}{F_{n,k}}
\renewcommand{\S}{S}

\newcommand{\de}{p}
\newcommand{\me}{\rho}

\title{Learning Manifolds with K-Means and K-Flats}

\author{
Guillermo D. Canas$^{\star,\dagger}$ \hspace*{0.3in} Tomaso Poggio$^{\star,\dagger}$ \hspace*{0.3in} Lorenzo A. Rosasco$^{\star,\dagger}$ \\
$\star$ Laboratory for Computational and Statistical Learning
 - MIT-IIT \\
$\dagger$ CBCL, McGovern Institute - Massachusetts Institute of Technology\\
\texttt{guilledc@mit.edu} \quad \texttt{tp@ai.mit.edu} \quad \texttt{lrosasco@mit.edu}
}

\begin{document}

\maketitle

\begin{abstract}
{
 We study the problem of  estimating  a manifold  from random samples.
In particular, we consider piecewise constant and piecewise linear estimators 
induced by k-means and k-flats, and analyze their performance.  
We extend previous results for k-means in two separate directions. 
First, we provide new results for k-means reconstruction on manifolds and, secondly, 
	we prove reconstruction bounds  for higher-order approximation (k-flats), 
	for which no known results were previously available. 
While the results for k-means are novel, some of the technical tools are well-established in the literature. 
In the case of k-flats, both the results and the mathematical tools are new. 
}


\end{abstract}

\section{Introduction}  \label{sec:setup}

%
 
Our study is  broadly motivated by questions in  high-dimensional learning.
As is well known,  learning in high dimensions  is feasible only if the data distribution satisfies suitable prior assumptions.
One such assumption is that the data distribution lies on, or is close to,  a   low-dimensional set 
embedded  in a high dimensional space,  for instance  a low dimensional  manifold.
This latter assumption  has proved to be useful in practice, as well as amenable to theoretical analysis, and it has 
led to a significant amount of recent work. 
 Starting from  \cite{isomap,RSLLE,beni03}, this set of ideas, broadly referred to as {\em manifold learning}, 
has been applied to  a variety of problems from supervised \cite{ste10} and semi-supervised learning \cite{be06}, 
to clustering \cite{vo07} and dimensionality reduction \cite{beni03}, to name a few.
%

Interestingly, the problem of learning the manifold itself has received less attention: 
given samples from a d-manifold $\MM$ embedded in some ambient space $\H$, the problem is 
 to  learn a set that approximates $\MM$ in a suitable sense.
 This problem has been considered in  computational geometry,  but in a setting in which typically 
the manifold   is a hyper-surface in a low-dimensional space (e.g.\ $\R^3$), 
	and the data  are typically not sampled probabilistically, see for instance~\cite{mls,Eigencrust}. 
The problem of learning a manifold  is also related to that of estimating the  support of a distribution,  
(see~\cite{cufr10,curo03} for recent surveys.) 
In this context, some of the distances  considered to measure approximation quality are the Hausforff distance, 
	and the so-called \emph{excess mass} distance. 

The reconstruction framework that we consider is related to the work of~\cite{Maggioni2011,Hari}, 
	as well as to the framework proposed in~\cite{mapo10}, 
	in which a manifold is approximated by a set, with performance measured by an expected distance to this set.
This setting is similar to the problem of dictionary learning (see for instance~\cite{Maial}, and extensive references therein), 
	in which a dictionary is found by minimizing a similar reconstruction error, 
	perhaps with additional constraints on an associated encoding of the data. 
%
Crucially, while the dictionary is learned on the empirical data, the quantity of interest is the expected reconstruction error, 
	which is the focus of this work. 
%

We analyze this problem by focusing on two important, and widely-used algorithms, namely k-means and k-flats. 
The k-means algorithm can be seen to define a piecewise constant approximation of $\MM$.
 Indeed, it   induces a Voronoi decomposition on  $\MM$, in which each Voronoi region 
is effectively approximated by a fixed mean. 
Given this, a natural extension is to consider higher order approximations, such as 
those induced by discrete collections of $k$ $d$-dimensional affine spaces (k-flats), with possibly better resulting performance. 
Since $\MM$ is a $d$-manifold,  
the k-flats approximation naturally resembles the way in which a manifold is locally approximated by its tangent bundle. 

Our analysis extends previous results  for k-means to  the case in which the data-generating distribution is supported on 
a manifold, and provides analogous results for k-flats. 
We note that the k-means algorithm has been widely studied, 
and thus much of our analysis in this case involves the combination of known facts to obtain novel results. 
The analysis of  k-flats, however, requires 
developing substantially new mathematical tools.

The rest of the paper is organized as follows. 
In section~\ref{sec:algo}, we describe the formal  setting and the algorithms that we study.
We begin our analysis by discussing the reconstruction properties of k-means in section~\ref{sec:disc}. 
In section~\ref{sec:results}, we present and discuss our main results, whose proofs are postponed to the appendices. 
\section{Learning Manifolds}\label{sec:algo}


Let $\H$ by a Hilbert space with
inner product  $\scal{\cdot}{\cdot}$,  endowed with a Borel probability 
measure $\me$  supported over a compact, smooth $d$-manifold $\MM$.
We assume the data to be given by a training set, in the form of samples 
$\X_n=(x_1,\dots,x_n)$ drawn identically and independently  with respect to $\me$.\\
Our goal is to {\em learn} a set $\Sn$ that approximates well the manifold.
The approximation (learning error) is measured by the expected reconstruction error
\begin{equation}\label{EErho}
\EE_\me(\Sn) := \ds{\int_{\MM}{d\me(x)\  d_{_{\H}}^2(x, \Sn)}}, 
\end{equation}
where the distance to a set $\S\subseteq \H$   is $d^2_{_{\H}}(x,\S) = \inf_{x'\in\S} d^2_{_{\H}}(x, x')$, 
with $d_{_{\H}}(x, x')=\nor{x-x'}$. 
This is the same reconstruction measure that has been the recent focus of~\cite{mapo10,Bartlett98theminimax,Hari}. 

It is easy to see that any set such that $\S\supset\MM$ will have zero risk, with $\MM$  being the ``smallest" such set (with respect to set containment.) In other words, the above error measure  does not introduce an explicit penalty on the ``size'' of $\Sn$: enlarging any \emph{given} $\Sn$ 
can never increase the learning error.\\
With this observation in mind, we study specific learning algorithms that, given the data, 
	produce a set belonging to some restricted hypothesis space $\hh$  (e.g.\ sets of size $k$ for k-means), 
	which effectively introduces a constraint on the size of the sets. 
Finally, note that the risk of Equation~\ref{EErho} is non-negative and, if the hypothesis space is sufficiently \emph{rich}, the risk of an unsupervised algorithm may converge to zero under suitable conditions. 

%
%
%
%
%

\subsection{Using K-Means and K-Flats for Piecewise Manifold Approximation}

In this work, we focus on two specific algorithms, namely k-means ~\cite{kmeans,Lloyd}  and k-flats~\cite{BradleyKflats}.
Although typically discussed in the Euclidean space case, their definition can be easily extended to a Hilbert space setting. 
The study of manifolds embedded in a Hilbert space is of special interest when considering non-linear (kernel) versions of the algorithms~\cite{KKM}. 
More generally, this setting can be 
seen as a limit case when  dealing with high  dimensional data. Naturally, the more classical 
setting of an absolutely continuous distribution over $d$-dimensional Euclidean space 
is simply a particular case, in which $\H=\mathbb{R}^d$, and $\MM$ is a domain with positive Lebesgue measure.

\noindent
{\bf K-Means}. 
Let $\hh=\mathcal{S}_k$ be the class of sets of size $k$ in $\H$. 
Given a training set  $\X_n$ and a choice of $k$,  k-means is defined by 
the minimization over $S\in\mathcal{S}_k$ of the empirical reconstruction error 
\begin{equation}\label{EEemp}
	\EE_n(S) := \frac{1}{n} \sum_{i=1}^n d_{_{\H}}^2(x_i,S). 
\end{equation}
where, for any fixed set $S$, $\EE_n(S)$  is an unbiased 
empirical estimate of $\EE_\me(S)$, so that k-means  can be seen to be performing a kind of empirical risk minimization 
\cite{Buhmann,BenDavid,mapo10,OTPOCIHS,mapo10}. 

A minimizer of Equation~\ref{EEemp} on $\mathcal{S}_k$ is a discrete set of $k$ {\em means} $\Snk=\{m_1,\dots,m_k\}$, 
which induces a Dirichlet-Voronoi tiling of $\H$: a collection of $k$ regions, each closest to a common mean~\cite{Aurenhammer}
(in our notation, the subscript $n$ denotes the dependence of $\Snk$ on the sample, while $k$ refers to its size.)
By virtue of $\Snk$ being a minimizing set, each mean must occupy the center of mass of the samples in its Voronoi region.  
These two facts imply that it is possible to compute a local minimum of the empirical risk by using a greedy 
coordinate-descent relaxation, namely Lloyd's algorithm~\cite{Lloyd}.  Furthermore,  given a finite sample $\X_n$, 
the number of locally-minimizing sets $\Snk$ is also finite since (by the center-of-mass condition) there cannot be more 
than the number of possible partitions of $\X_n$ into $k$ groups, and therefore the global minimum must be attainable. 
Even though Lloyd's algorithm provides no guarantees of closeness to the global minimizer, in practice 
it is possible to use a randomized approximation algorithm, such as kmeans++~\cite{kmpp}, 
which provides guarantees of approximation to the global minimum in expectation with respect to the randomization.  

\noindent{\bf K-Flats}. 
Let $\hh=\mathcal{F}_k$ be  the class of collections of $k$ {\em flats} (affine spaces) of dimension $d$. 
For any value of $k$, k-flats, analogously to k-means,  aims at finding 
the set $\Fk\in\mathcal{F}_k$ that minimizes the empirical reconstruction~\eqref{EEemp} over  ${\cal F}_k$.
By an argument similar to the one used for k-means,  a global minimizer must be attainable, 
and a Lloyd-type relaxation converges to a local minimum. Note that, in this case, given a 
Voronoi partition of $\MM$ into regions closest to each $d$-flat, new optimizing flats for that partition can 
be computed by a $d$-truncated PCA solution on the samples falling in each region. 

\subsection{Learning a Manifold with K-means and K-flats}

In practice, k-means is often interpreted to be a  clustering  algorithm, with clusters defined by the Voronoi diagram of the set of means $S_{n,k}$.
In this interpretation, Equation~\ref{EEemp} is simply rewritten by summing over the Voronoi regions, and adding 
all pairwise distances between samples in the region (the intra-cluster distances.)
For instance, this point of view is considered in~\cite{buhmerd} 
where k-means is studied from an information theoretic persepective.
K-means can also be interpreted to be performing vector quantization, where the goal is to minimize the 
	encoding error associated to a nearest-neighbor quantizer~\cite{GershoGray}. 
Interestingly, in the limit of increasing sample size, this problem coincides, in a precise sense~\cite{PollardKMC}, 
	with the problem of optimal quantization of probability distributions
(see for instance the excellent monograph of~\cite{GrafLushgyMonograf}.)


When the data-generating distribution is supported on a manifold $\MM$, k-means can be seen to be approximating 
	points on the manifold by a discrete set of means. 
Analogously to the Euclidean setting, this induces a Voronoi decomposition of  $\MM$, in which each Voronoi region 
		is effectively approximated by a fixed mean (in this sense k-means produces a piecewise constant approximation of $\MM$.) 
As in the Euclidean setting, the limit of this problem with increasing sample size is precisely the problem of optimal quantization of distributions on manifolds, 
	which is the subject of significant recent work in the field of optimal quantization~\cite{GruberOQ,gruber2007convex}.

In this paper,  we take the above view of k-means as defining a (piecewise constant) approximation of the manifold $\MM$ 
	supporting the data distribution. 
In particular, we are interested in the behavior of the expected reconstruction error $\EE_\me(\Snk)$, for varying $k$ and $n$. 
This perspective has an interesting relation with dictionary learning, 
	in which one is interested in finding a dictionary, and an associated representation, 
	that allows to approximately reconstruct a finite set of data-points/signals. 
In this interpretation, the set of means can be seen as a dictionary of size $k$  
that produces a maximally sparse representation (the k-means encoding), see for example~\cite{Maial} and references therein.
Crucially, while the dictionary is learned on the available empirical data, the quantity of interest is the expected reconstruction error, and the question of characterizing the performance with respect to this latter quantity naturally arises.
%
%
%

Since k-means produces a piecewise constant approximation of the data, 
a natural idea is to  consider higher orders of approximation, such as 
approximation by discrete collections of $k$ $d$-dimensional affine spaces (k-flats), with possibly better performance. 
Since $\MM$ is a $d$-manifold,  
the approximation induced by k-flats may more naturally resemble the way in which a manifold is locally approximated by its tangent bundle. 
We provide in Sec.~\ref{sec:kflats_rates} a partial answer to this question. 

\section{Reconstruction Properties of k-Means}\label{sec:disc}

Since 
we are interested in 
the behavior of the expected reconstruction~\eqref{EErho} of k-means and k-flats for {\em varying $k$ and $n$}, 
before analyzing this behavior, 
we consider 
what is currently known about this problem, based on previous work. 
While k-flats is a relatively new algorithm whose behavior is not yet well understood, 
several properties of k-means are currently known. 

%

Recall that k-means find an discrete set $\Snk$  of size $k$ that best approximates the samples in the sense of~\eqref{EEemp}. 
Clearly, as $k$  increases, the empirical reconstruction error $\EE_n(\Snk)$ 
cannot increase, and typically decreases.  
However, we are ultimately interested in the expected reconstruction error, 
	and therefore would like to understand the behavior of $\EE_\me(\Snk)$ with varying $k,n$. 

In the context of optimal quantization, 
	the behavior of the expected reconstruction 
error $\EE_\me$ has been considered  for an approximating set $\Sk$  
	obtained by minimizing the \emph{expected} reconstruction error itself over the 
hypothesis space $\hh={\cal S}_k$. 
The set $\Sk$ can thus be interpreted  as the output of a {\em population}, or infinite sample version of k-means.
In this case, it is possible to show that $\EE_\me(\Sk)$ is a non increasing function of $k$ and, in fact,  to derive explicit rates. 
For example in the case $\H=\R^d$, and under fairly general technical assumptions, 
	it is possible to show that $\EE_\me(S_k) = \Theta(k^{-2/d})$, where the constants depend on $\rho$ and $d$~\cite{GrafLushgyMonograf}. 

In machine learning, the properties of k-means have been studied, {\em for  fixed $k$}, by considering the 
{\em excess} reconstruction error $\EE_\me(\Snk)-\EE_\me(\Sk)$. 
In particular, this quantity has been studied for $\H=\R^d$, and shown to be, 
with high probability, 
of order $\sqrt{kd/n}$, up-to logarithmic factors~\cite{mapo10}. 
The case where 
$\H$ is a Hilbert space has been considered in~\cite{mapo10,OTPOCIHS}, where an upper-bound of order $k/\sqrt{n}$ is proven to hold with high probability.  The more general setting where $\H$ is a metric space has been studied in~\cite{BenDavid}.
\begin{figure}[t]
\begin{center}
\begin{minipage}{0.45\textwidth}
\centerline{\tiny{Sphere Dataset}}
\includegraphics[width=\columnwidth]{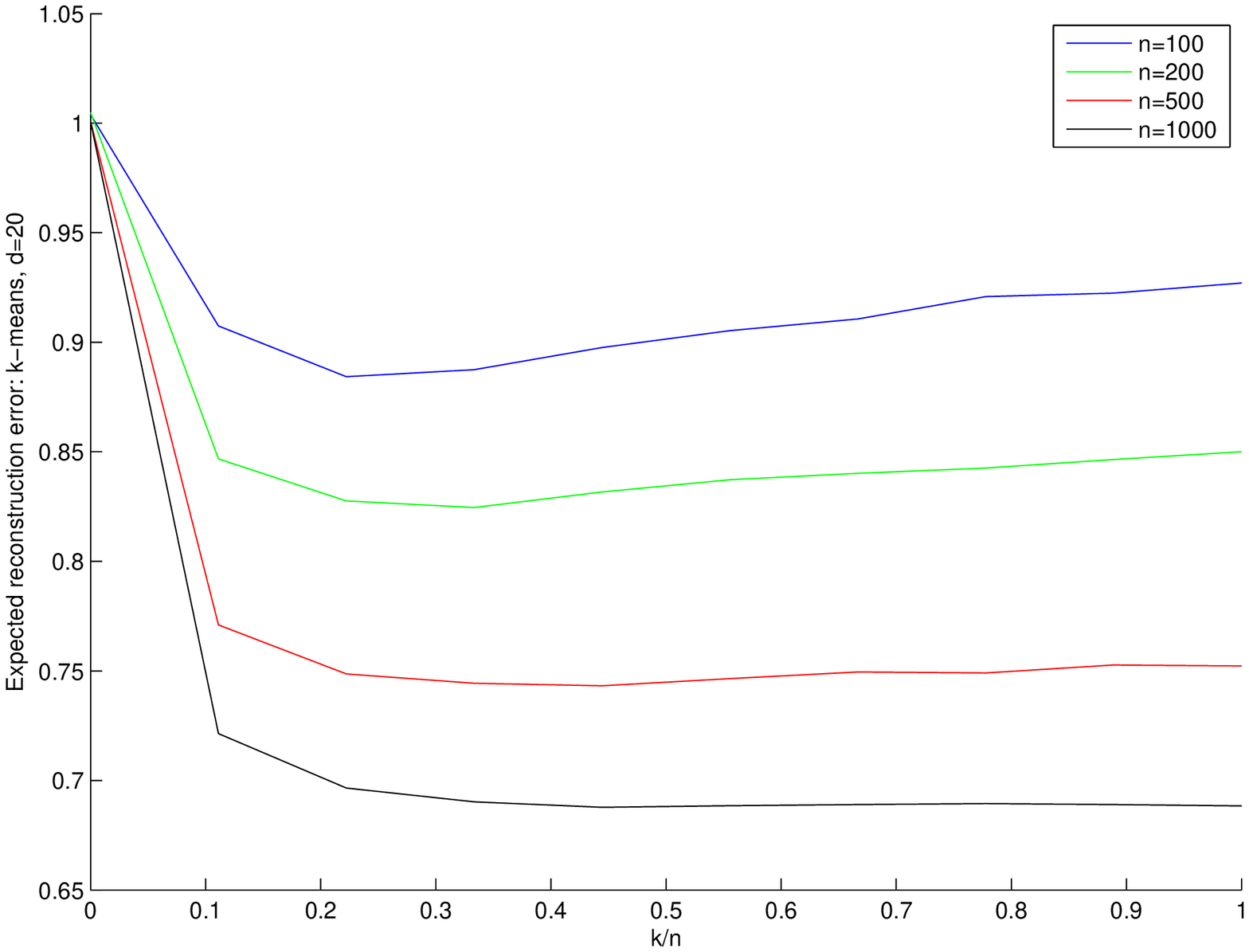}
\end{minipage}
\begin{minipage}{0.45\textwidth}
\centerline{\tiny{MNIST Dataset }}
\includegraphics[width=\columnwidth]{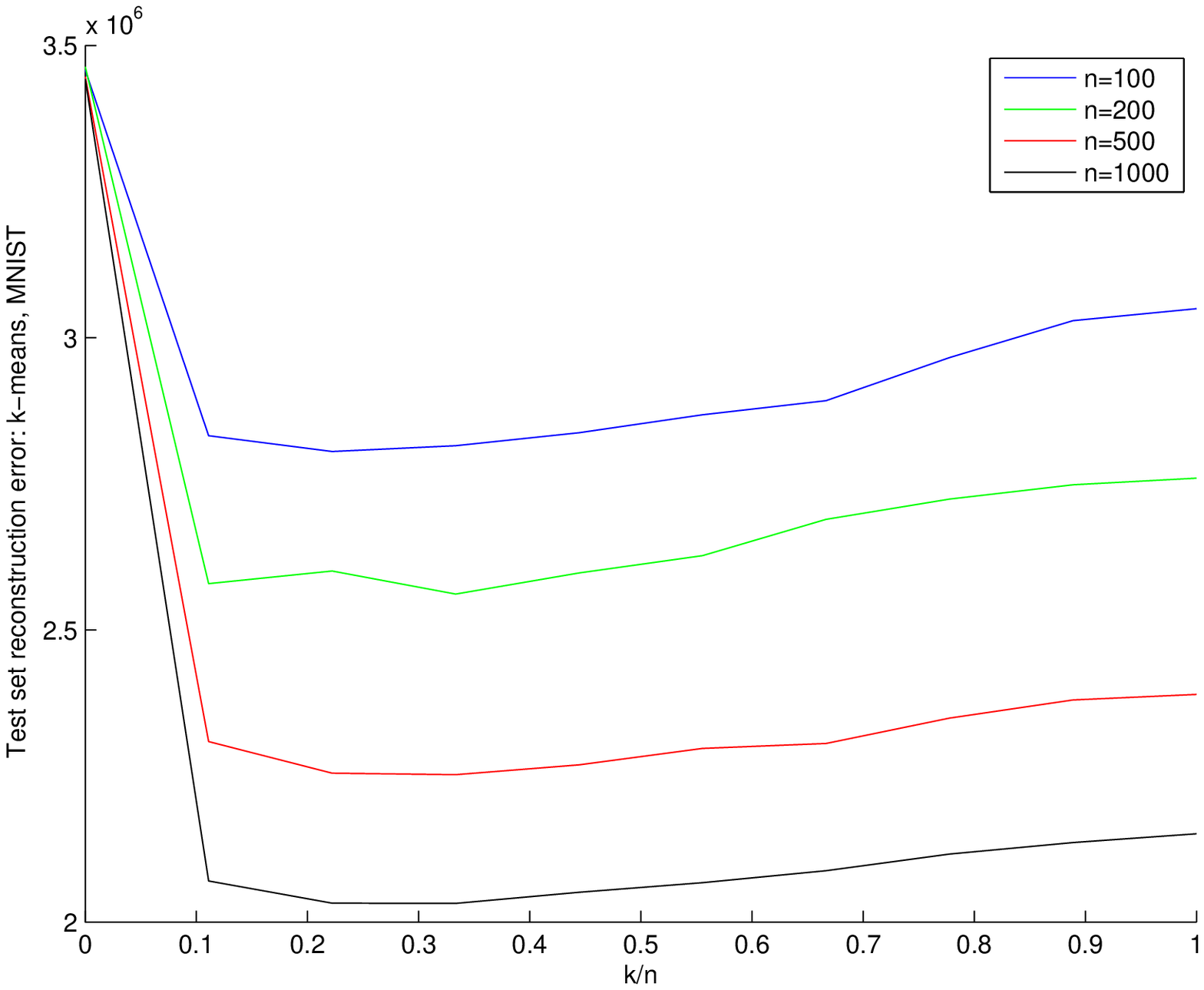}
\end{minipage}
\caption{{
We consider the behavior of  k-means  for   data sets obtained by sampling uniformly a $19$ dimensional sphere embedded in $\R^{20}$ (left).
For each value of $k$,  k-means (with k-means++ seeding) is run $20$ times, 
and the best solution kept.  
The reconstruction performance on a (large) hold-out  set is reported as a function of $k$.
The results for four different training set cardinalities are reported: for small number of points, the reconstruction error 
decreases sharply for small $k$ and then increases,  while it is simply decreasing for larger data sets. 
A similar experiment, yielding similar results, is performed on subsets of the MNIST (\texttt{http://yann.lecun.com/exdb/mnist}) database (right).
In this case the data might be thought to be concentrated around a low dimensional manifold. For example 
\cite{HeiAud05} report an average intrinsic dimension $d$ for each digit to be between $10$ and $13$. }}
\label{Fig:1}
\end{center}
\end{figure}

When analyzing the behavior of $\EE_\me(\Snk)$, 
and in the particular case that $\H=\R^d$, the above results can be combined to obtain, 
with high probability, a bound   of the form 
\begin{align}\label{basic}
 \EE_\me(\Snk) &\le |\EE_\me(\Snk) - \EE_n(\Snk)| + \EE_n(\Snk) - \EE_n(\Sk) + |\EE_n(\Sk) - \EE_\me(\Sk)|+\EE_\me(\Sk) \notag\\
 			&\le C\left ( \sqrt{\frac{kd}{n}} +k^{-2/d} \right) 
\end{align}
up to logarithmic factors, where the constant $C$ does not depend on $k$ or $n$ (a complete derivation is given in the Appendix.) 
The above inequality  suggests a somewhat surprising effect: the expected reconstruction 
properties of  k-means may be described by a {\em trade-off} between a statistical error (of order $\sqrt{\frac{kd}{n}}$)
and a geometric approximation error (of order $k^{-2/d}$.)  

The existence of such a tradeoff between the approximation, 
	and the statistical errors may itself not be entirely obvious, see the discussion in~\cite{Bartlett98theminimax}. 
For instance, in the k-means problem, it is intuitive that, as more means are inserted, the expected distance from a random sample to the means should decrease, 
and one might expect a similar behavior for the expected reconstruction error.
This observation naturally begs the question of whether and when  this trade-off really exists  or if it is simply a result of the looseness in the bounds. 
In particular, one could ask how tight the bound~\eqref{basic} is. 

While the bound on $\EE_\me(\Sk)$ is known to be tight for $k$ sufficiently large~\cite{GrafLushgyMonograf}, 
 the remaining terms (which are dominated by $|\EE_\me(\Snk) - \EE_n(\Snk)|$) are derived by controlling the supremum of an empirical process
\begin{equation}\label{eqsup}
	\sup_{S\in\Sk} | \EE_n(S) - \EE_\rho(S)|
\end{equation}
and it is unknown whether available bounds for it are tight~\cite{mapo10}. 
Indeed, it is not clear how close the \emph{distortion redundancy} $\EE_\rho(\Snk) - \EE_\rho(\Sk)$ 
	is to its known lower bound of order $d\sqrt{\frac{k^{1-\frac{4}{d}}}{n}}$ (in expectation)  ~\cite{Bartlett98theminimax}. 
More importantly, we are not aware of a lower bound for  $\EE_\rho (\Snk)$ itself.
Indeed, as pointed out in~\cite{Bartlett98theminimax}, ``The exact dependence of the minimax distortion redundancy on k and d is still a challenging open problem".

Finally, we note that, whenever a trade-off can be shown to hold, it may be 
used to justify a heuristic for choosing $k$ empirically as the value that minimizes the reconstruction error in a hold-out set. 




In Figure~\ref{Fig:1} we perform some simple numerical simulations showing that the trade-off indeed occurs in certain regimes. 
The following  example provides a situation where a trade-off can be easily shown to occur.

\begin{figure}[t]
\begin{center}
\subfigure[$\EE_\me(S_{k=1})\simeq 1.5$]{\label{fig:tradeoffa}\includegraphics[width=3.5cm]{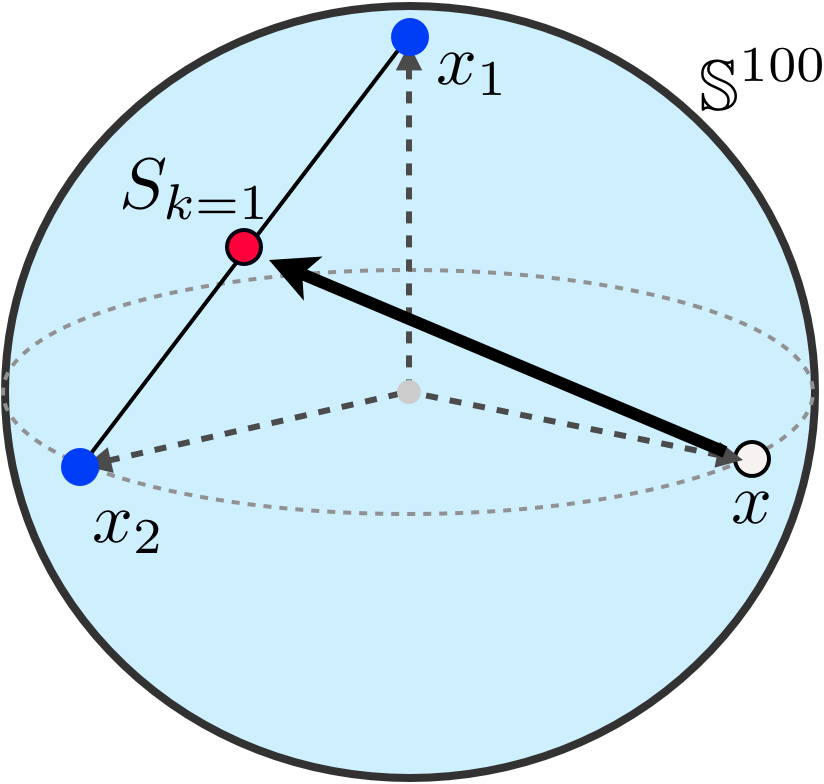}}\quad\quad\quad
\subfigure[$\EE_\me(S_{k=2})\simeq 2$]{\label{fig:tradeoffb}\includegraphics[width=3.5cm]{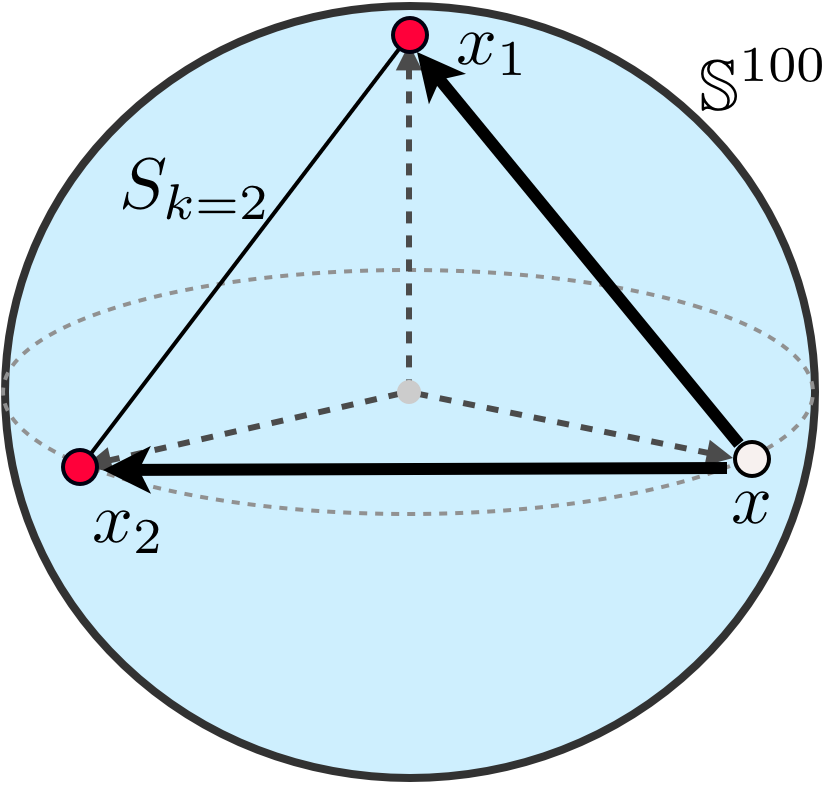}}
\caption{{\small 
The optimal k-means (red) computed from $n=2$ samples drawn uniformly on $\mathbb{S}^{100}$ (blue.) 
For a) $k=1$, the expected squared-distance to a random point $x\in\mathbb{S}^{100}$ is $\EE_\me(S_{k=1})\simeq 1.5$, while for b) $k=2$, 
	it is 
	$\EE_\me(S_{k=2})\simeq 2$. }}
\label{default}
\end{center}
\end{figure}

\begin{ex}
Consider a setup in which $n=2$ samples are drawn from a uniform distribution on the unit $d=100$-sphere, 
	though the argument holds for other $n$ much smaller than $d$. 
Because $d \gg n$, with high probability, the samples are nearly orthogonal: $<x_1,x_2>_\H\simeq 0$, 
while a third sample $x$ drawn uniformly on $\mathbb{S}^{100}$ will also very likely be nearly orthogonal to both $x_1,x_2$~\cite{ledoux2001concentration}. 
The k-means solution on this dataset is clearly $S_{k=1}=\{(x_1+x_2)/2\}$ (Fig~\ref{fig:tradeoffa}). 
Indeed, since $S_{k=2}=\{x_1,x_2\}$ (Fig~\ref{fig:tradeoffb}), 
it is $\EE_\me(S_{k=1})\simeq 1.5 < 2 \simeq \EE_\me(S_{k=2})$ with very high probability. 
In this case, it is better to place a single mean closer to the origin (with $\EE_\me(\{0\})=1$), 
than to place two means at the sample locations.
This example is sufficiently simple that the exact k-means solution is known, but 
the effect can be observed in more complex settings. 
\end{ex}

\section{Main Results}\label{sec:results}

\noindent{\bf Contributions}. 
Our work extends previous results in two different directions:
\begin{enumerate}[(a)]
\item We provide an  analysis of k-means for the case in which the data-generating distribution is supported on 
a manifold embedded in a Hilbert space. In particular, in this setting: 1) we derive new results on the approximation error,
and 2) new sample complexity 
results (learning rates) arising from the choice of $k$ by optimizing the resulting bound. 
We analyze the case in which a solution is obtained
from an approximation algorithm, such as k-means++~\cite{kmpp}, to include this computational error in the bounds. 

\item We generalize the above results from k-means to k-flats, deriving learning rates obtained from new   bounds on both the  statistical
and the approximation errors.  To the best of our knowledge, these results provide the first theoretical analysis  of k-flats in either sense. 
\end{enumerate}

We note that the k-means algorithm has been widely studied in the past, 
and much of our analysis in this case involves the combination of known facts to obtain novel results. However, in the case of k-flats, there is currently no known analysis, and we provide novel results  
	as well as new performance bounds for each of the components in the bounds.

Throughout this section we make the following technical assumption:
\begin{ass}\label{ass0}
$\MM$ is a smooth d-manifold with metric of class $\mathcal{C}^1$, contained in the unit ball in $\H$, and with  volume measure denoted by $\mu_\text{I}$. 
The probability measure   $\rho$ is absolutely continuous with respect to $\mu_\text{I}$, with density $\de$.
\end{ass}
\subsection{Learning Rates for k-Means}
The first result considers the idealized case where we have access to an exact solution for k-means.
\begin{theorem}\label{thkm}
	Under Assumption~\ref{ass0}, if $\Snk$ is a solution of k-means  then, for $0 < \delta < 1$, there are constants $C$ and $\gamma$ dependent only on $d$, and sufficiently large $n'$ such that, 
by setting
	\begin{equation}
		k_n = n^{\frac{d}{2(d+2)}} \cdot \left(\frac{C}{24\sqrt{\pi}}\right)^{d/(d+2)} \cdot \left\{\ds{\int_{\MM} d\mu_{\text{I}}(x) \de(x)^{d/(d+2)} }\right\},
	\end{equation}
	and $\Sn=\S_{n,k_n}$, it is
	\begin{equation}\label{eqKM}
		\PP\left[ \EE_\me(\Sn) \le \gamma \cdot n^{-1/(d+2)} \cdot \sqrt{\ln 1/\delta} \cdot  \left\{\ds{\int_{\MM} d\mu_{\text{I}}(x) \de(x)^{d/(d+2)} }\right\}  \right] \ge 1-\delta ,
	\end{equation}
	for all $n\ge n'$, where $C\sim d/(2\pi e)$ and $\gamma$ grows sublinearly with $d$. 
\end{theorem}

\begin{rem}
Note that the distinction between distributions with density in $\MM$, 
and singular distributions is important. 
The bound of Equation~\eqref{eqKM} holds only when the absolutely continuous part of $\me$ over $\MM$ is non-vanishing. 
	the case in which  the distribution is singular over $\MM$ requires a different analysis, and may result in faster convergence rates. 
\end{rem}

%
%
	The following result considers the case where the k-means++ algorithm is used to compute the estimator.

\begin{theorem}\label{thkmpp}
	Under Assumption~\ref{ass0}, if $\Snk$ is the solution of   k-means++ , 
		then for $0<\delta <1$, there are constants $C$ and $\gamma$ that depend only on $d$, and a sufficiently large $n'$ such that,  
by setting
	\begin{equation}
		k_n = n^{\frac{d}{2(d+2)}} \cdot \left(\frac{C}{24\sqrt{\pi}}\right)^{d/(d+2)} \cdot \left\{\ds{\int_{\MM} d\mu_{\text{I}}(x) \de(x)^{d/(d+2)} }\right\},
	\end{equation}
and $\Sn=\S_{n,k_n}$,
it is
	\begin{equation}\label{eqKMPP}
		\PP\left[ \E_Z \EE_\me(\Sn) \le \gamma \cdot n^{-1/(d+2)} \left( \ln n + \ln \|\de\|_{d/(d+2)} \right) \cdot \sqrt{\ln 1/\delta} \cdot  \left\{\displaystyle{\int_{\MM} d\mu_{\text{I}}(x) \de(x)^{d/(d+2)} }\right\}  \right] \ge 1-\delta ,
	\end{equation}
	for all $n\ge n'$, 
	where the expectation is with respect to the random choice $Z$ in the algorithm, and $\|\de\|_{d/(d+2)} = \left\{\displaystyle{\int_{\MM} d\mu_{\text{I}}(x) \de(x)^{d/(d+2)} }\right\}^{(d+2)/d}$, 
	$C\sim d/(2\pi e)$, and $\gamma$ grows sublinearly with $d$. 
\end{theorem}
\begin{rem}
In the particular case that $\H=\mathbb{R}^d$ and $\MM$ is contained in the unit ball, 
we may further bound the distribution-dependent part of Equations~\ref{eqKM} and~\ref{eqKMPP}. 
Using H\"older's inequality, 
one obtains
\begin{equation}\label{eq:wd}
\begin{split}
	\displaystyle{\int d\nu(x) \de(x)^{d/(d+2)} } &\le 
		\left[\displaystyle{\int_{\MM} d\nu(x)   \de(x)  }\right]^{d/(d+2)} \cdot \left[\displaystyle{\int_{\MM} d\nu(x)}\right]^{2/(d+2)} \\
		&\le \text{Vol}(\MM)^{2/(d+2)} \le \omega_d^{2/(d+2)}, 
\end{split}
\end{equation}
where $\nu$ is the Lebesgue measure in $\mathbb{R}^d$, and $\omega_d$ is the volume of the $d$-dimensional unit ball. 

It is clear from the proof of Theorem~\ref{thkm} that, in this case, we may choose
\[
	k_n = n^{\frac{d}{2(d+2)}} \cdot \left(\frac{C}{24\sqrt{\pi}}\right)^{d/(d+2)} \cdot \omega_d^{2/d},
\]
independently of the density $\de$, to obtain a 
bound $\EE_\me(\Sn^*) = O\left(  n^{-1/(d+2)}  \cdot \sqrt{\ln 1/\delta}\right)$ with probability $1-\delta$ (and similarly for Theorem~\ref{thkmpp}, except for an additional $\ln n$ term), 
where the constant only depends on the dimension. 
\end{rem}

\begin{rem}
Note that according to the above theorems, choosing $k$ requires knowledge of properties 
of the distribution $\rho$ underlying the data, such as the intrinsic 
dimension of the support. In fact, following the ideas in \cite{stch08} Section 6.3-5,  
it is easy to prove  that choosing $k$ to minimize the reconstruction error  on a hold-out set,
allows to achieve  the same learning rates (up to a logarithmic factor), adaptively in the sense that knowledge of 
properties of $\rho$ are not needed.
\end{rem}

\subsection{Learning Rates for k-Flats}\label{sec:kflats_rates}

To study  k-flats,  we need to slightly strengthen Assumption~\ref{ass0} by adding to it by the following:

\begin{ass}\label{asskf}
Assume   the manifold $\MM$ to have  metric of class $\mathcal{C}^3$, 
	and finite second fundamental form $\text{II}$~\cite{carmo1992riemannian}.
 \end{ass}
One reason for the higher-smoothness assumption is that k-flats uses higher order approximation, whose analysis requires a higher order of differentiability. \\
We begin by providing a result for k-flats on hypersurfaces (codimension one), and next extend it to manifolds in more general spaces.

\begin{theorem}\label{thkf}
Let, $\H=\mathbb{R}^{d+1}$.
Under Assumptions~\ref{ass0},\ref{asskf}, 
if $\Fnk$ is a solution of k-flats,
then there is a constant $C$ that depends only on $d$, and sufficiently large $n'$ such that, by setting 	
	\begin{equation}
		k_n =  n^{\frac{d}{2(d+4)}} \cdot \left(  \frac{C}{2\sqrt{2\pi d}}  \right)^{d/(d+4)}  \cdot \left(\kappa_{_\MM}\right)^{4 / (d+4)} ,
	\end{equation}
	and  $\Fn=\F_{n,k_n}$,
then for all $n\ge n'$ it is
	\begin{equation}\label{eqKF}
		\PP\left[ \EE_\me(\Fn)  \le   2\left(8\pi d\right)^{2/(d+4)} C^{d/(d+4)} \cdot n^{-2/(d+4)}  \cdot \sqrt{\frac{1}{2}\ln 1/\delta} \cdot   \left(\kappa_{_\MM}\right)^{4/(d+4)}   \right] \ge 1-\delta ,
	\end{equation}
where	
	$
		\kappa_{_\MM} :=  
		\mu_{_{|\text{\emph{II}}|}}(\MM) = {\int_{_\MM }  d\mu_{_{\text{\emph{I}}}}(x) |\kappa^{1/2}_G(x)|  }
	$
	 is the total root curvature of $\MM$, $\mu_{|\text{\emph{II}}|}$ is the measure associated with the (positive) second fundamental form, and $\kappa_{_G}$ is the Gaussian curvature on $\MM$.
\end{theorem} 


In the more general case of a $d$-manifold $\MM$ (with metric in $\mathcal{C}^3$) embedded in a separable Hilbert space $\H$, we cannot make any assumption on the codimension of $\MM$ (the dimension of the orthogonal complement to the tangent space at each point.)  
In particular, the second fundamental form $\text{II}$, which is an extrinsic quantity describing how the tangent spaces bend locally 
is, at every $x\in\MM$, a map $\text{II}_x: T_x\MM\mapsto\left(T_x\MM\right)^\perp$ 
(in this case of class $\mathcal{C}^1$ by Assumption~\ref{asskf}) 
from the tangent space to its orthogonal complement ($\text{II}(x) := B(x,x)$ in the notation of~\cite[p.\ 128]{carmo1992riemannian}.)
Crucially, in this case, we may no longer assume the dimension of the orthogonal complement $\left(T_x\MM\right)^\perp$ to be finite. \\
Denote by 
$|\text{II}_x| = \sup_{\substack{r\in T_x\MM \\ \|r\|\le 1}} \nor{\text{II}_x(r)}_{_{\H}}$,  the operator norm of $\text{II}_x$. We have:

\begin{theorem}\label{thkfg}
Under Assumptions~\ref{ass0},\ref{asskf}, 
if $\Fnk$ is a solution to the  k-flats problem,
then there is a constant $C$ that depends only on $d$, and sufficiently large $n'$ such that, by setting 	
	\begin{equation}
		k_n =  n^{\frac{d}{2(d+4)}} \cdot \left(  \frac{C}{2\sqrt{2\pi d}}  \right)^{d/(d+4)}  \cdot \kappa_{_\MM}^{4 / (d+4)} ,
	\end{equation}
		and  $\Fn=\F_{n,k_n}$,
then for all $n\ge n'$ it is
	\begin{equation}\label{eqKFG}
		\PP\left[ \EE_\me(\Fn)  \le   2\left(8\pi d\right)^{2/(d+4)} C^{d/(d+4)} \cdot n^{-2/(d+4)}  \cdot \sqrt{\frac{1}{2}\ln 1/\delta} \cdot   \kappa_{_\MM}^{4/(d+4)}   \right] \ge 1-\delta ,
	\end{equation}
where	
	$
		\kappa_{_\MM} :=  
		{\int_{_\MM }  d\mu_{\text{I}}(x)  \text{ } |\text{\emph{II}}_x|^2 }
	$
\end{theorem} 

	Note that the better k-flats bounds stem from the higher approximation power of $d$-flats over points. 
	Although this greatly complicates the setup and proofs, as well as the analysis of the constants, 
		the resulting bounds are of order $O\left(n^{-2/(d+4)}\right)$, compared with the slower order $O\left(n^{-1/(d+2)}\right)$ of k-means.
%



\subsection{Discussion}



%


In all the results, the final performance does not depend on the dimensionality of the embedding space (which in fact can be infinite), but only 
on the intrinsic dimension of the space on which the data-generating distribution is defined. 
%
The key to these results is an approximation construction in which the Voronoi regions on the manifold (points closest to a given mean or flat) 
	are guaranteed to have vanishing diameter in the limit of $k$ going to infinity. 
Under our construction, 
a hypersurface is approximated efficiently by tracking the variation of its tangent spaces by using the second fundamental form. Where this form vanishes, the Voronoi regions of an approximation will not be ensured to have vanishing diameter with $k$ going to infinity, unless certain care is taken in the analysis. 

An important point of interest is that the approximations are controlled by averaged quantities, such as the total root curvature (k-flats for surfaces of codimension one), total curvature (k-flats in arbitrary codimensions), and $d/(d+2)$-norm of the probability density (k-means), 
	which are integrated over the domain where the distribution is defined. 
Note that these types of quantities have been linked to provably tight approximations in certain cases, such as for convex manifolds~\cite{GruberConvexI,enets}, in contrast with worst-case methods that place a constraint on a maximum curvature, or minimum injectivity radius (for instance~\cite{Maggioni2011,Hari}.) 
Intuitively, it is easy to see that a constraint on an average quantity may be arbitrarily less restrictive than one on its maximum. 
A small difficult region (e.g.\ of very high curvature) may cause the bounds of the latter to substantially degrade, 
	while the results presented here would not be adversely affected so long as the region is small. 

Additionally, care has been taken throughout to analyze the behavior of the constants. In particular, there are no constants in the analysis that grow exponentially with the dimension, and in fact, many have polynomial, or slower growth. We believe this to be an important point, since this ensures that the asymptotic bounds do not hide an additional exponential dependence on the dimension.

\bibliographystyle{plain}
{\small \bibliography{BIBReconstructionNips2012v3G}}

\begin{thebibliography}{10}

\bibitem{Maggioni2011}
William~K Allard, Guangliang Chen, and Mauro Maggioni.
\newblock Multiscale geometric methods for data sets ii: Geometric
  multi-resolution analysis.
\newblock {\em Applied and Computational Harmonic Analysis}, 1:1--38, 2011.

\bibitem{kmNP3}
Daniel Aloise, Amit Deshpande, Pierre Hansen, and Preyas Popat.
\newblock Np--hardness of euclidean sum-of-squares clustering.
\newblock {\em Mach. Learn.}, 75:245--248, May 2009.

\bibitem{kmpp}
David Arthur and Sergei Vassilvitskii.
\newblock k--means++: the advantages of careful seeding.
\newblock In {\em Proceedings of the eighteenth annual ACM-SIAM symposium on
  Discrete algorithms}, SODA '07, pages 1027--1035, Philadelphia, PA, USA,
  2007. SIAM.

\bibitem{Aurenhammer}
Franz Aurenhammer.
\newblock Voronoi diagrams: A survey of a fundamental geometric data structure.
\newblock {\em ACM Comput. Surv.}, 23:345--405, September 1991.

\bibitem{Bartlett98theminimax}
Peter~L. Bartlett, Tamas Linder, and Gabor Lugosi.
\newblock The minimax distortion redundancy in empirical quantizer design.
\newblock {\em IEEE Transactions on Information Theory}, 44:1802--1813, 1998.

\bibitem{BartlettRG}
Peter~L. Bartlett and Shahar Mendelson.
\newblock Rademacher and gaussian complexities: Risk bounds and structural
  results.
\newblock {\em Journal of Machine Learning Research}, 3:463--482, 2002.

\bibitem{beni03}
M.~Belkin and P.~Niyogi.
\newblock Laplacian eigenmaps for dimensionality reduction and data
  representation.
\newblock {\em Neural Comput.}, 15(6):1373--1396, 2003.

\bibitem{be06}
Mikhail Belkin, Partha Niyogi, and Vikas Sindhwani.
\newblock Manifold regularization: a geometric framework for learning from
  labeled and unlabeled examples.
\newblock {\em J. Mach. Learn. Res.}, 7:2399--2434, 2006.

\bibitem{BenDavid}
Shai Ben-David.
\newblock A framework for statistical clustering with constant time
  approximation algorithms for k-median and k-means clustering.
\newblock {\em Mach. Learn.}, 66(2-3):243--257, March 2007.

\bibitem{OTPOCIHS}
G{\'e}rard Biau, Luc Devroye, and G{\'a}bor Lugosi.
\newblock On the performance of clustering in hilbert spaces.
\newblock {\em IEEE Transactions on Information Theory}, 54(2):781--790, 2008.

\bibitem{kPCAbounds}
Gilles Blanchard, Olivier Bousquet, and Laurent Zwald.
\newblock Statistical properties of kernel principal component analysis.
\newblock {\em Mach. Learn.}, 66:259--294, March 2007.

\bibitem{BradleyKflats}
P.~S. Bradley and O.~L. Mangasarian.
\newblock k-plane clustering.
\newblock {\em J. of Global Optimization}, 16:23--32, January 2000.

\bibitem{Buhmann}
Joachim~M. Buhmann.
\newblock Empirical risk approximation: An induction principle for unsupervised
  learning.
\newblock Technical report, University of Bonn, 1998.

\bibitem{buhmerd}
Joachim~M. Buhmann.
\newblock Information theoretic model validation for clustering.
\newblock In {\em International Symposium on Information Theory, Austin Texas}.
  IEEE, 2010.
\newblock (in press).

\bibitem{Chernaya}
E~V Chernaya.
\newblock On the optimization of weighted cubature formulae on certain classes
  of continuous functions.
\newblock {\em East J. Approx}, 1995.

\bibitem{enets}
Kenneth~L. Clarkson.
\newblock Building triangulations using $\epsilon$-nets.
\newblock In {\em Proceedings of the thirty-eighth annual ACM symposium on
  Theory of computing}, STOC '06, pages 326--335, New York, NY, USA, 2006. ACM.

\bibitem{cufr10}
A.~Cuevas and R.~Fraiman.
\newblock Set estimation.
\newblock In {\em New perspectives in stochastic geometry}, pages 374--397.
  Oxford Univ. Press, Oxford, 2010.

\bibitem{curo03}
A.~Cuevas and A.~Rodr{\'{\i}}guez-Casal.
\newblock Set estimation: an overview and some recent developments.
\newblock In {\em Recent advances and trends in nonparametric statistics},
  pages 251--264. Elsevier B. V., Amsterdam, 2003.

\bibitem{kmNP2}
Sanjoy Dasgupta and Yoav Freund.
\newblock Random projection trees for vector quantization.
\newblock {\em IEEE Trans. Inf. Theor.}, 55:3229--3242, July 2009.

\bibitem{KKM}
Inderjit~S. Dhillon, Yuqiang Guan, and Brian Kulis.
\newblock Kernel k-means: spectral clustering and normalized cuts.
\newblock In {\em Proceedings of the tenth ACM SIGKDD international conference
  on Knowledge discovery and data mining}, KDD '04, pages 551--556, New York,
  NY, USA, 2004. ACM.

\bibitem{dieudonne2008foundations}
J.~Dieudonne.
\newblock {\em Foundations of Modern Analysis}.
\newblock Pure and Applied Mathematics. Hesperides Press, 2008.

\bibitem{carmo1992riemannian}
M.P. DoCarmo.
\newblock {\em Riemannian geometry}.
\newblock Theory and Applications Series. Birkh{\"a}user, 1992.

\bibitem{GershoGray}
Allen Gersho and Robert~M. Gray.
\newblock {\em Vector quantization and signal compression}.
\newblock Kluwer Academic Publishers, Norwell, MA, USA, 1991.

\bibitem{GrafLushgyMonograf}
Siegfried Graf and Harald Luschgy.
\newblock {\em Foundations of quantization for probability distributions}.
\newblock Springer-Verlag New York, Inc., Secaucus, NJ, USA, 2000.

\bibitem{GruberConvexI}
P.~M. Gruber.
\newblock Asymptotic estimates for best and stepwise approximation of convex
  bodies i.
\newblock {\em Forum Mathematicum}, 15:281--297, 1993.

\bibitem{GruberOQ}
Peter~M. Gruber.
\newblock Optimum quantization and its applications.
\newblock {\em Adv. Math}, 186:2004, 2002.

\bibitem{gruber2007convex}
P.M. Gruber.
\newblock {\em Convex and discrete geometry}.
\newblock Grundlehren der mathematischen Wissenschaften. Springer, 2007.

\bibitem{HeiAud05}
Matthias Hein and Jean-Yves Audibert.
\newblock Intrinsic dimensionality estimation of submanifolds in rd.
\newblock In {\em ICML '05: Proceedings of the 22nd international conference on
  Machine learning}, pages 289--296, 2005.

\bibitem{isomap}
V.~De~Silva J.~B.~Tenenbaum and J.~C. Langford.
\newblock A global geometric framework for nonlinear dimensionality reduction.
\newblock {\em Science}, 290(5500):2319--2323, 2000.

\bibitem{Eigencrust}
Ravikrishna Kolluri, Jonathan~Richard Shewchuk, and James~F. O'Brien.
\newblock Spectral surface reconstruction from noisy point clouds.
\newblock In {\em Proceedings of the 2004 Eurographics/ACM SIGGRAPH symposium
  on Geometry processing}, SGP '04, pages 11--21, New York, NY, USA, 2004. ACM.

\bibitem{ledoux2001concentration}
M.~Ledoux.
\newblock {\em The Concentration of Measure Phenomenon}.
\newblock Mathematical Surveys and Monographs. American Mathematical Society,
  2001.

\bibitem{mls}
David Levin.
\newblock Mesh-independent surface interpolation.
\newblock In Hamann Brunnett and Mueller, editors, {\em Geometric Modeling for
  Scientific Visualization}, pages 37--49. Springer-Verlag, 2003.

\bibitem{Lloyd}
Stuart~P. Lloyd.
\newblock Least squares quantization in pcm.
\newblock {\em IEEE Transactions on Information Theory}, 28:129--137, 1982.

\bibitem{kmeans}
J.~B. MacQueen.
\newblock Some methods for classification and analysis of multivariate
  observations.
\newblock In L.~M.~Le Cam and J.~Neyman, editors, {\em Proc. of the fifth
  Berkeley Symposium on Mathematical Statistics and Probability}, volume~1,
  pages 281--297. University of California Press, 1967.

\bibitem{kmNP1}
Meena Mahajan, Prajakta Nimbhorkar, and Kasturi Varadarajan.
\newblock The planar k--means problem is np-hard.
\newblock In {\em Proceedings of the 3rd International Workshop on Algorithms
  and Computation}, WALCOM '09, pages 274--285, Berlin, Heidelberg, 2009.
  Springer-Verlag.

\bibitem{Maial}
Julien Mairal, Francis Bach, Jean Ponce, and Guillermo Sapiro.
\newblock Online dictionary learning for sparse coding.
\newblock In {\em Proceedings of the 26th Annual International Conference on
  Machine Learning}, ICML '09, pages 689--696, 2009.

\bibitem{mapo10}
A.~Maurer and M.~Pontil.
\newblock K--dimensional coding schemes in hilbert spaces.
\newblock {\em IEEE Transactions on Information Theory}, 56(11):5839 --5846,
  nov. 2010.

\bibitem{Hari}
Hariharan Narayanan and Sanjoy Mitter.
\newblock Sample complexity of testing the manifold hypothesis.
\newblock In {\em Advances in Neural Information Processing Systems 23}, pages
  1786--1794. MIT Press, 2010.

\bibitem{PollardKMC}
David Pollard.
\newblock Strong consistency of k-means clustering.
\newblock {\em Annals of Statistics}, 9(1):135--140, 1981.

\bibitem{RSLLE}
ST~Roweis and LK~Saul.
\newblock Nonlinear dimensionality reduction by locally linear embedding.
\newblock {\em Science}, 290:2323--2326, 2000.

\bibitem{Slepian}
{D}avid {S}lepian.
\newblock {T}he one-sided barrier problem for {G}aussian noise.
\newblock {\em {B}ell {S}ystem {T}ech. {J}.}, 41:463--501, 1962.

\bibitem{ste10}
Florian Steinke, Matthias Hein, and Bernhard Sch{\"o}lkopf.
\newblock Nonparametric regression between general {R}iemannian manifolds.
\newblock {\em SIAM J. Imaging Sci.}, 3(3):527--563, 2010.

\bibitem{stch08}
I.~Steinwart and A.~Christmann.
\newblock {\em Support vector machines}.
\newblock Information Science and Statistics. Springer, New York, 2008.

\bibitem{FejesT}
G~Fejes Toth.
\newblock Sur la representation d\'{ }une population in par une nombre d\'{
  }elements.
\newblock {\em Acta Math. Acad. Sci. Hungaricae}, 1959.

\bibitem{vo07}
Ulrike von Luxburg.
\newblock A tutorial on spectral clustering.
\newblock {\em Stat. Comput.}, 17(4):395--416, 2007.

\bibitem{Zador}
Paul~L. Zador.
\newblock Asymptotic quantization error of continuous signals and the
  quantization dimension.
\newblock {\em IEEE Transactions on Information Theory}, 28(2):139--148, 1982.

\end{thebibliography}


\appendix

\section{Methodology and Derivation of Results}\label{sec:derivation}

Although both k-means and k-flats optimize the same empirical risk, 
the performance measure we are interested in is that of Equation~\ref{EErho}. We may bound it from above as follows: 
\begin{eqnarray} \label{eqsplit}
	\EE_\me(\Snk) &\le |\EE_\me(\Snk) - \EE_n(\Snk)| + 
					\EE_n(\Snk) - \EE_n(S^*_k) + |\EE_n(S^*_k) - \EE^*_{\me,k}| + \EE^*_{\me,k} \\
		&\le  2 \cdot \underbrace{ \displaystyle{\sup_{S\in {\cal S}_k} |\EE_\me(S) - \EE_n(S)|}}_{\text{Statistical error}} + \underbrace{\EE^*_{\me,k}}_{\text{Approximation error}} \label{eqsplit2}
\end{eqnarray}
where $\EE^*_{\me,k} := \inf_{S\in {\cal S}_k} \EE_\me(S)$  is the best attainable 
performance over  ${\cal S}_k$, and $S^*_k$ is a set for which the best performance is attained. 
Note that $\EE_n(\Snk) - \EE_n(S^*_k) \le 0$ by the definition of $\Snk$. 
The same error decomposition can be considered for k-flats, by replacing $\Snk$ 
by $\Fnk$ and  ${\cal S}_k$ by ${\cal F}_k$.


Equation~\ref{eqsplit} decomposes the total learning error into two terms: a uniform (over all sets in the class $C_k$) bound on the difference between the empirical, and true error measures, 
	and an \emph{approximation error} term. The uniform statistical error bound will depend on the samples, and thus may hold with a certain probability. 

In this setting, the approximation error will typically tend to zero as the class $C_k$ becomes larger (as $k$ increases.) 
Note that this is true, for instance, if 
	$C_k$ is the class of discrete sets of size $k$, as in the k-means problem. 

The performance of Equation~\ref{eqsplit} is, through its dependence on the samples, a random variable. 
We will thus set out to find probabilistic bounds on its performance, as a function of the number $n$ of samples, and the size $k$ of the approximation.  
	By choosing the approximation size parameter $k$ to minimize these bounds, we obtain performance bounds 
	as a function of the sample size. 

\section{K-Means}\label{sec:KM}

We use the above decomposition to derive sample complexity bounds for the performance of the k-means algorithm.
To derive explicit bounds on the different error terms we  have to combine in a novel way  some previous results 
and  some new observations.\\

\noindent{\bf Approximation error}. 
The error $\EE^*_{\me,k} = \inf_{\Sk\in\SS_k} \EE_\me(\Sk)$ is related to the problem of optimal quantization. 
The classical optimal quantization problem is quite well understood, going back to the fundamental work of~\cite{Zador,FejesT} on optimal quantization for data transmission, and more recently by the work of~\cite{GrafLushgyMonograf,gruber2007convex,GruberOQ,Chernaya}. 
In particular, it is known that, for distributions  with finite moment of order $2+\lambda$, for some $\lambda>0$, it is~\cite{GrafLushgyMonograf}
\begin{equation}\label{km:ae}
	\displaystyle{\lim_{k\rightarrow\infty}  \EE^*_{\me,k} \cdot k^{2/d}} = 
	C  \left\{\displaystyle{\int d\nu(x) \de_a(x)^{d/(d+2)} } \right\}^{(d+2)/d}
\end{equation}
where $\nu$ is the Lebesgue measure, $\de_a$ is the density of the absolutely continuous part of the distribution (according to its Lebesgue decomposition), and $C$ is a constant that depends only on the dimension. 
Therefore, the approximation error decays \emph{at least} as fast as $k^{-2/d}$. 


We note that, by setting $\mu$ to be the uniform distribution over the unit cube $[0,1]^d$, it clearly is
\[ \displaystyle{\lim_{k\rightarrow\infty} \EE^*_{\mu,k} \cdot k^{2/d} } = C \]
and thus, by making use of Zador's asymptotic formula~\cite{Zador}, and combining it with a result of B\"or\"oczky (see~\cite{gruber2007convex}, p.\ 491), 
	we observe that $C \sim \left(d / (2\pi e)\right)^{r/2}$ with $d\rightarrow\infty$, for the $r$-th order quantization problem. 
In particular, this shows that the constant $C$ only depends on the dimension, and, in our case  ($r=2$), has only linear growth in $d$, a fact that will be used in the sequel.

The approximation error $\EE^*_{\me,k} = \inf_{\Sk\in\SS_k} \EE_\me(\Sk)$ of k-means 
is  related to the problem of optimal quantization on manifolds, for which some results 
are known~\cite{GruberOQ}. 
By calling $\EE^*_{\MM,\de,k}$ the approximation error only among sets of means contained in $\MM$, Theorem~\ref{ThGruber} in Appendix ~\ref{sec:KF},  implies in this case (letting $r=2$) that 
\begin{equation}\label{eq:aekm}
	\lim_{k\rightarrow\infty}\EE^*_{\me,k}\cdot  k^{2/d} = C \left\{ \ds{  \int_\MM d\mu_{_\text{I}}(x) \text{ } \de(x)^{d/(d+2)}} \right\}^{(d+2)/d} 
\end{equation}
where $\de$ is absolutely continuous over $\MM$ and, 
by replacing $\MM$ with a $d$-dimensional domain in $\mathbb{R}^d$, it is clear that the constant $C$ is the same as above. 

Since restricting the means to be on $\MM$ cannot decrease the approximation error, it is $\EE^*_{\me,k} \le \EE^*_{\MM,\de,k}$, and therefore the right-hand side of Equation~\ref{eq:aekm} provides an (asymptotic) upper bound to $\EE^*_{\me,k} \cdot k^{2/d}$.

For the statistical error we use available bounds.\\
\noindent{\bf Statistical error}. 
The statistical error of Equation~\ref{eqsplit}, which uniformly bounds the difference between the empirical, and expected  error, has been widely-studied in recent years in the literature~\cite{mapo10,Hari,Bartlett98theminimax}. 
In particular, it has been shown that, for a distribution $\de$ over the unit ball in $\mathbb{R}^d$, it is
\begin{equation}\label{km:se}
	\displaystyle{\sup_{\S\in\SS_k} |\EE_\me(\S) - \EE_n(\S)|} \le \frac{k\sqrt{18\pi}}{\sqrt{n}} + \sqrt{\frac{8 \ln 1/\delta }{ n }}
\end{equation}
with probability $1-\delta$~\cite{mapo10}. 
Clearly, this implies convergence $\EE_n(\S) \rightarrow \EE_\me(\S)$ almost surely, as $n\rightarrow\infty$; although this latter result was proven earlier in~\cite{PollardKMC}, under the less restrictive condition that $p$ have finite second moment. 
\\
By bringing together the above results, we obtain the bound in Theorem~\ref{thkm} on the performance of k-means, whose proof is postponed to Appendix A. \\

Further, we can consider the error incurred by the actual optimization algorithm used to compute the k-means solution.\\
\noindent{\bf Computational error}. 
In practice, the k-means problem is NP-hard~\cite{kmNP3,kmNP2,kmNP1}, with  the original Lloyd relaxation algorithm providing no guarantees of closeness to the global minimum of Equation~\ref{EEemp}. 
However, practical approximations, such as the k-means++ algorithm~\cite{kmpp}, exist. 
When using k-means++, means are inserted one by one at samples selected with probability proportional to their squared distance to the set of previously-inserted means. 
This randomized seeding has been shown  by~\cite{kmpp} to output a set that is, in expectation, within a  $8\left(\ln k + 2\right)$-factor of the optimal. 
Once again, by combining these results, we obtain Theorem~\ref{thkmpp}, 
whose proof is also in 
Appendix A. 

We use the results discussed in Section~\ref{sec:derivation} to obtain the proof of Theorem~\ref{thkm} as follows. 

\begin{proof}
Letting $\|\de\|_{d/(d+2)} := \left\{ \ds{\int d\mu_{\text{I}}(x) \de(x)^{d/(d+2)} } \right\}^{(d+2)/d}$, then 
with probability $1-\delta$, it is
\begin{equation}\label{eq:derivkm}
\begin{split}
	\EE_\me(\Snk) &\le 2 n^{-1/2} \left( k\sqrt{18\pi} + \sqrt{8\ln 1/\delta}\right) + C k^{-2/d}  \cdot\|\de\|_{d/(d+2)} \\
		&\le  2 n^{-1/2} k\sqrt{18\pi} \cdot \sqrt{8\ln 1/\delta} +  C k^{-2/d} \cdot\|\de\|_{d/(d+2)} \\
		&=   24\sqrt{\pi} k n^{-1/2} \sqrt{\ln 1/\delta} + C k^{-2/d} \cdot\|\de\|_{d/(d+2)} \\
		&= 2  \sqrt{\ln 1/\delta}  n^{-1/(d+2)}  C^{d/(d+2)} \left(24\sqrt{\pi}\right)^{2/(d+2)} \cdot \left\{ \ds{\int d\mu_{\text{I}}(x) \de(x)^{d/(d+2)} } \right\}
\end{split}
\end{equation}
where the parameter
\begin{equation}\label{kmkn}
	k_n = n^{\frac{d}{2(d+2)}} \cdot \left(\frac{C}{24\sqrt{\pi}}\right)^{d/(d+2)} \cdot \left\{\ds{\int d\mu_{\text{I}}(x) \de(x)^{d/(d+2)} }\right\}  
\end{equation}
has been chosen to balance the summands in the third line of Equation~\ref{eq:derivkm}. 
\end{proof}

\noindent The proof of Theorem~\ref{thkmpp} follows a similar argument.

\begin{proof}
In the case of Theorem~\ref{thkmpp}, the additional multiplicative term $A_k = 8(\ln k + 2)$ corresponding to the computational error incurred by the k-means++ algorithm does not affect the choice of parameter $k_n$ since both summands in the third line of Equation~\ref{eq:derivkm} are multiplied by $A_k$ in this case. Therefore, we may simply use the same choice of $k_n$ as in Equation~\ref{kmkn} in this case to obtain
\begin{equation}\label{eq:derivkmpp}
\begin{split}
	\E_Z \EE_\me(\Snk) &\le 2 n^{-1/2} \left( k\sqrt{18\pi} + \sqrt{8\ln 1/\delta}\right) + C k^{-2/d}  \cdot\|\de\|_{d/(d+2)} \cdot 8(\ln k+2) \\
		&\le 16  \sqrt{\ln 1/\delta}  n^{-1/(d+2)}  C^{d/(d+2)} \left(24\sqrt{\pi}\right)^{2/(d+2)} \cdot \left\{ \ds{\int d\mu_{\text{I}}(x) \de(x)^{d/(d+2)} } \right\} \\
		& \cdot \left[ 2 + \frac{d}{d+2} \left(  \frac{1}{2} \ln n + \ln \frac{C}{12\sqrt\pi} + \ln \| \de\|_{d/(d+2)} \right) \right]
\end{split}
\end{equation}
with probability $1-\delta$, where the expectation is with respect to the random choice $Z$ in the algorithm. 
From this the bound of Theorem~\ref{thkmpp} follows. 
\end{proof}

\section{K-Flats}\label{sec:KF}

Here we state a series of lemma that we prove in the next section.
For the k-flats problem, we begin by introducing a uniform bound on the difference between empirical (Equation~\ref{EEemp}) 
and expected risk (Equation~\ref{EErho}.) 

\begin{lemma}\label{kfse}
If $\mathcal{F}_k$ is the class of sets of $k$ $d$-dimensional affine spaces then, 
	with probability $1-\delta$ on the sampling of $\X_n\sim p$, it is
\[ 	\ds{\sup_{\X'\in\mathcal{F}_k} |\EE_\me(\X') -  \EE_{n}(\X')} | \le  k \sqrt{\frac{2\pi d}{n}} + \sqrt{\frac{\ln 1/\delta}{2n}} \]
\end{lemma}

By combining the above result with 
approximation error bounds, we may produce performance bounds on the expected risk for the k-flats problem, 
	with appropriate choice of parameter $k_n$. 
We distinguish between the codimension one hypersurface case, and the more general case of a smooth manifold $\MM$ embedded in a Hilbert space. 
We begin with an approximation error bound for hypersurfaces in Euclidean space.

\begin{lemma}\label{kfae}
Assume given $\MM$ smooth with metric of class $\mathcal{C}^3$ in $\mathbb{R}^{d+1}$.
	If $\mathcal{F}_k$ is the class of sets of $k$ $d$-dimensional affine spaces, and $\EE^*_{\me,k}$ is the minimizer of Equation~\ref{EErho} over $\mathcal{F}_k$, then there is a constant $C$ that depends on $d$ only, such that
	\begin{equation*}
		\ds{\lim_{k\rightarrow\infty} \EE^*_{\me,k} \cdot k^{4/d} } \le C  \cdot \left(\kappa_{\MM}\right)^{4/d}   
	\end{equation*}
	where $\kappa_\MM := \mu_{|\text{II}|}(\MM)$ is the total root curvature of $\MM$,  and $\mu_{|\text{II}|}$ is the measure associated with the (positive) second fundamental form.
	The constant $C$ grows as $C\sim \left( d / (2\pi e)\right)^2$ with $d\rightarrow\infty$. 
\end{lemma}

For the more general problem of approximation of a smooth manifold in a separable Hilbert space, we begin by
considering the definitions in Section~\ref{sec:results} 
the second fundamental form $\text{II}$ and its operator norm $|\text{II}_q|$ at a point $q\in\MM$. The we have:

\begin{lemma}\label{kfaepp}
Assume given a $d$-manifold $\MM$ with metric in $\mathcal{C}^3$ embedded in a separable Hilbert space $\H$. 
	If $\mathcal{F}_k$ is the class of sets of $k$ $d$-dimensional affine spaces, and $\EE^*_{\me,k}$ is the minimizer of Equation~\ref{EErho} over $\mathcal{F}_k$, then there is a constant $C$ that depends on $d$ only, such that

	\begin{equation*}
		\ds{ \lim_{k\rightarrow\infty} \EE^*_{\me,k} \cdot k^{4/d} } \le  C  \cdot \left(\kappa_{\MM}\right)^{4/d}
	\end{equation*}
where	
	$
		\kappa_{\MM} :=  
		{\int_{_\MM }  d\mu_{\text{I}}(x)  \text{ } \frac 1 4 |\text{II}_x|^2 }
	$
	and $\mu_{\text{I}}$ is the volume measure over $\MM$.  
	The constant $C$ grows as $C\sim \left( d / (2\pi e)\right)^2$ with $d\rightarrow\infty$. 
\end{lemma}

We combine these two results into Theorems~\ref{thkf} and~\ref{thkfg}, 
whose derivation is in Appendix B.

\subsection{Proofs}

We begin proving  the bound on the statistical error given in Lemma~\ref{kfse}.
\begin{proof}
We begin by finding uniform upper bounds on the difference between Equations~\ref{EErho} and~\ref{EEemp} for the class 
$\mathcal{F}_k$ of sets of $k$ $d$-dimensional affine spaces. 
To do this, we will first bound the Rademacher complexity $\mathcal{R}_n(\mathcal{F}_k, \de)$ of the class $\mathcal{F}_k$. 

Let $\Phi$ and $\Psi$ be Gaussian processes indexed by $\mathcal{F}_k$, and defined by
\begin{equation}\begin{split}
	 \Phi_{\X'} &= \ds{\sum_{i=1}^{n} \gamma_i \min_{j=1}^{k} d_{_{\H}}^2(x_i, \pi'_j x_i) }  \\
	 \Psi_{\X'} &= \ds{\sum_{i=1}^{n} \gamma_i \sum_{j=1}^{k} d_{_{\H}}^2(x_i, \pi'_j x_i) } 
\end{split}\end{equation}
$\X'\in\mathcal{F}_k$,  $\X'$ is the union of $k$ $d$-subspaces: $\X'=\cup_{j=1}^k F_j$, where each $\pi'_j$ is an orthogonal projection onto $F_j$, and  $\gamma_i$ are independent Gaussian sequences of zero mean and unit variance. 

Noticing that $d_{_{\H}}^2(x, \pi x) = \|x\|^2 - \|\pi x\|^2 = \|x\|^2 - \left< x x^t, \pi \right>_{_F}$ for any orthogonal projection $\pi$ (see for instance~\cite{kPCAbounds}, Sec. 2.1), 
	where $\left<\cdot,\cdot\right>_{_F}$ is the Hilbert-Schmidt  inner product, 
we may verify that:
\begin{equation}\begin{split}\label{eq:slepian}
	 \E_\gamma \left( \Phi_{\X'} - \Phi_{\X''} \right)^2 &= \ds{\sum_{i=1}^n \left[ \min_{j=1}^k \|x_i\|^2 - \left<x_i x_i^t, \pi'_j\right>_{_F}  -  \left(\min_{j=1}^k \|x_i\|^2 - \left<x_i x_i^t, \pi''_j\right>_{_F} \right) \right]^2} \\ 
	 	&\le  \ds{\sum_{i=1}^n  \max_{j=1}^k \left(   \left< x_i x_i^t, \pi'_j\right>_{_F}  -  \left< x_i x_i^t, \pi''_j\right>_{_F}  \right)^2 } \\
		&\le  \ds{\sum_{i=1}^n  \sum_{j=1}^k \left(   \left< x_i x_i^t, \pi'_j\right>_{_F}  -  \left< x_i x_i^t, \pi''_j\right>_{_F}  \right)^2 }  = \E_\gamma \left(\Psi_{\X'} - \Psi_{\X''} \right)^2
\end{split}\end{equation}

Since it is,
\begin{equation}\begin{split}\label{eq:gc}
	\E_{\gamma} \ds{\sup_{\X'\in\mathcal{F}_k} \sum_{i=1}^n \gamma_i \sum_{j=1}^k \left< x_i x_i^t, \pi'_j \right>_{_F} } &= 
				\E_{\gamma} \ds{\sup_{\X'\in\mathcal{F}_k}  \sum_{j=1}^k \left< \sum_{i=1}^n \gamma_i x_i x_i^t, \pi'_j \right>_{_F} } \\
				&\le k \E_{\gamma} \ds{\sup_{\pi} \left< \sum_{i=1}^n \gamma_i x_i x_i^t, \pi \right>_{_F} } \\
					&\le k \text{ } \ds{\sup_{\pi} \|\pi\|_{_F} } \E_\gamma \|\ds{\sum_{i=1}^n \gamma_i x_i x_i^t }\|_{_F} \le k \sqrt{d n}
\end{split}\end{equation}
we may bound the Gaussian complexity $\Gamma_n(\mathcal{F}_k, \de)$ as follows:
\begin{equation}\begin{split}
	\Gamma_n(\mathcal{F}_k, \de) &= \frac{2}{n} \E_\gamma  \ds{\sup_{\X'\in\mathcal{F}_k} \sum_{i=1}^n \gamma_i \min_{j=1}^k d_{_{\H}}^2(x_i, \pi'_j x_i) } \\
			&\le \frac{2}{n} \E_\gamma  \ds{\sup_{\X'\in\mathcal{F}_k} \sum_{i=1}^n \gamma_i \sum_{j=1}^k  \left< x_i x_i^t, \pi'_j \right>_{_F} }  \le 2 k \sqrt{\frac{d}{n}}
\end{split}\end{equation}
where the first inequality follows from Equation~\ref{eq:slepian} and Slepian's Lemma~\cite{Slepian}, and the second from Equation~\ref{eq:gc}. 

Therefore the Rademacher complexity is bounded by
\begin{equation}\begin{split}
	\mathcal{R}_n(\mathcal{F}_k, \de) &\le \sqrt{\pi/2} \Gamma_n(\mathcal{F}_k, \de) \le k \sqrt{\frac{2 \pi d}{n}}
\end{split}\end{equation}

Finally, by Theorem 8 of~\cite{BartlettRG}, it is:
\begin{equation}\begin{split}
	\ds{\sup_{\X'\in\mathcal{F}_k} |\EE_\me(\X') -  \EE_{n}(\X')} | \le   \mathcal{R}_n(\mathcal{F}_k, \de)  + \sqrt{\frac{\ln 1/\delta}{2 n}} \le  k \sqrt{\frac{2 \pi d}{n}} + \sqrt{\frac{\ln 1/\delta}{2 n}}
\end{split}\end{equation}
as desired. 
\end{proof}

\subsection{Approximation Error}

In order to prove approximation bounds for the k-flats problem, 
	we will begin by first considering the simpler setting of a smooth $d$-manifold in $\mathbb{R}^{d+1}$ space 
	(codimension $1$), and later we will extend the analysis to the general case.

\subsection*{Approximation Error: Codimension One}

Assume that it is $\H=\mathbb{R}^{d+1}$ with the natural metric, and $\MM$ is a compact, smooth $d$-manifold with metric of class $\mathcal{C}^2$. 
Since $\MM$ is of codimension one, the second fundamental form at each point is a map from the tangent space to the reals. 
Assume given $\alpha > 0$ and $\lambda > 0$. 
At every point $x\in\MM$, define the metric $Q_x := |\text{II}_x| + \alpha'(x) \text{I}_x$, where
\begin{itemize}
	\item[a)] $\text{I}$ and $\text{II}$ are, respectively, the first and second fundamental forms on $\MM$~\cite{carmo1992riemannian}. 
	\item[b)] $|\text{II}|$ is the \emph{convexified} second fundamental form, whose eigenvalues are those of $\text{II}$ but in absolute value. If the second fundamental form $\text{II}$ is written in coordinates (with respect to an orthonormal basis of the tangent space) 
	as $S \Lambda S^T$, with $S$ orthonormal, and $\Lambda$ diagonal, 
		then $|\text{II}|$ is $S |\Lambda| S^T$ in coordinates.
		Because $|\text{II}|$ is continuous and positive semi-definite, it has an associated measure $\mu_{|\text{II}|}$
		(with respect to the volume measure $\mu_{\text{I}}$.) 
	\item[c)]  $\alpha'(x) > 0$ is chosen such that $d\mu_{_{Q_x}}/d\mu_{\text{I}} = d\mu_{|\text{II}|}/d\mu_{\text{I}} + \alpha$. 
	Note that such $\alpha'(x) > 0$ always exists since: 
		\begin{itemize}
			\item[$\cdot$] $\alpha'(x)=0$ implies  $d\mu_{_{Q_x}}/d\mu_{\text{I}} = d\mu_{|\text{II}|}/d\mu_{\text{I}}$, and
			\item[$\cdot$] $d\mu_{_{Q_x}}/d\mu_{\text{I}}$ 
			can be made arbitrarily large by increasing $\alpha'(x)$. 
		\end{itemize}
		and therefore there is some intermediate value of $\alpha'(x)>0$ that satisfies the constraint. 
\end{itemize}
In particular, from condition c), it is clear that $Q$ is everywhere positive definite. 

Let $\mu_\text{I}$ and $\mu_{_Q}$ be the measures over $\MM$, associated with $\text{I}$ and $Q$. 
Since, by its definition, $\mu_{\text{II}}$ is absolutely continuous with respect to $\text{I}$, then so must $Q$ be.
Therefore, we may define 
\[	
	\omega_{_Q} := d\mu_{_Q} / d\mu_\text{I}
\]
to be the density of $\mu_{_Q}$ with respect to $\mu_{\text{I}}$. 

Consider the discrete set $P_k\subset\MM$ of size $k$ that minimizes the quantity 
\begin{equation}\label{eqOQ4}
	f_{Q,\de}(P_k) =  \displaystyle{ \int_{_\MM} d\mu_{_Q}(x) \left[ \frac{\de(x)}{\omega_{_Q}(x)} \right] \min_{p\in P_k} d^4_{_{_Q}}(x,p)  }
\end{equation}
among all sets of $k$ points \emph{on} $\MM$. 
$f_{Q,\de}(P_k)$ is the (fourth-order) quantization error over $\MM$, with metric $Q$, and with respect to a weight function $\de/\omega_{_Q}$. 
Note that, in the definition of $f_{Q,\de}(P_k)$, it is crucial that the measure ($\mu_{_{Q}}$), and distance ($d_{_{Q}}$) match, in the sense that $d_{_Q}$ is the geodesic distance with respect to the metric $Q$, 
	whose associated measure is $\mu_{_Q}$.

The following theorem, adapted from~\cite{GruberOQ}, characterizes the relation between $k$ and the quantization error $f_{Q,\de}(P_k)$ on a Riemannian manifold. 

\begin{theorem}\emph{[\cite{GruberOQ}]}\label{ThGruber}
	Given a smooth compact Riemannian $d$-manifold $\MM$ with metric $Q$ of class $\mathcal{C}^1$, and a continuous function $w:\MM\rightarrow\RR^+$, then
	\begin{equation}\label{GruberOQ}
		 \ds{  \min_{P\in\mathcal{P}_k}  \int_{\MM} d\mu_{_Q}(x) w(x)  \min_{p\in P}  d^r_{_Q}(x,p)  }  \sim  C  \left\{ \ds{  \int_\MM d\mu_{_Q}(x) w(x)^{d/(d+r)} } \right\}^{(d+r)/d} \cdot k^{-r/d}
	\end{equation}
	as $k\rightarrow\infty$, 
	where the constant $C$ depends only on $d$. 
	
	Furthermore, for each connected $\MM$, there is a number $\xi>1$ such that each set $P_k$ that minimizes Equation~\ref{GruberOQ} is a $\left(k^{-1/d} / \xi\right)$-packing and $\left(\xi  k^{-1/d}\right)$-cover of $\MM$, with respect to $d_{_Q}$. 
\end{theorem}

This last result, which shows that a minimizing set $P_k$ of size $k$ must be a $\left(\xi  k^{-1/d}\right)$-cover, clearly implies, by the definition of Voronoi diagram and the triangle inequality, the following key corollary. 

\begin{corollary}\label{cor:diam}
	Given $\MM$, there is $\xi > 1$ such that each set $P_k$ that minimizes Equation~\ref{GruberOQ} has Voronoi regions of diameter no larger than $2\xi k^{-1/d}$, as measured by the distance $d_{_Q}$. 
\end{corollary}

Let each $P_k\subset\MM$ be a minimizer of Equation~\ref{eqOQ4} of size $k$, 
	then, for each $k$, define $F_k$ to be the union of ($d$-dimensional affine) tangent spaces to $\MM$ at each $q\in P_k$, that is, 
	$F_k := \cup_{q\in P_k} T_q\MM$. 
We may now use the definition of $P_k$ to bound the approximation error $\EE_\me(F_k)$ on this set. 

We begin by establishing some results that link distance to tangent spaces on manifolds to the geodesic distance $d_{_{Q}}$ associated with $Q$. 
The following lemma appears (in a slightly different form) as Lemma 4.1 in~\cite{enets}, and is borrowed from~\cite{GruberOQ,GruberConvexI}. 

\begin{lemma}\emph{[\cite{GruberOQ,GruberConvexI},~\cite{enets}]}\label{lem:lambda}
	Given $\MM$ as above, and $\lambda > 0$ then, for every $p\in\MM$ there is an open neighborhood $V_\lambda(p)\ni p$ in $\MM$ such that, for all $x,y\in V_\lambda(p)$, it is
	\begin{equation}\label{eq:lambda}
		 d^2_{_{\H}}(x, T_y\MM) \le (1+\lambda) d^4_{|\text{II}|}(x,y) 
	\end{equation}
	where $ d_{_{\H}}(x, T_y\MM)$ is the distance from $x$ to the tangent plane $T_y\MM$ at $y$, 
		and $d_{|\text{II}|}$ is the geodesic distance associated with the convexified second fundamental form. 
\end{lemma}

From the definition of $Q$, it is clear that, because $Q$ strictly dominates $|\text{II}|$ then, for points $x,y$ satisfying the conditions of Equation~\ref{eq:lambda}, 
it must be $d_{_{\H}}(x, T_y\MM) \le (1+\lambda)  d_{|\text{II}|}(x,y)  \le (1+\lambda) d_{_Q}(x,y)$.

Given our choice of $\lambda > 0$, Lemma~\ref{lem:lambda} implies 
that there is a collection of $k$ neighborhoods, centered around the points $p\in P_k$, 
	such that Equation~\ref{eq:lambda} holds inside each. 
However, these neighborhoods may be too small for our purposes. 
In order to apply Lemma~\ref{lem:lambda} to our problem, we will need to prove a stronger condition. 
We begin by considering the Dirichlet-Voronoi regions $D_{_{\MM,Q}}(p; P_k)$  of points $p\in P_k$, with respect to the distance $d_{_Q}$. 
That is, 
	\[ D_{_{\MM,Q}}(p; P_k) = \{x\in\MM : d_{_{Q}}(x,p) \le d_{_{Q}}(x,q), \forall q\in P_k\} \]
where, as before, $P_k$ is a set of size $k$ minimizing Equation~\ref{eqOQ4}. 

\begin{lemma}\label{lem:cover}	
	For each $\lambda > 0$, there is $k'$ such that, for all $k\ge k'$, and all $q\in P_k$,  Equation~\ref{eq:lambda} holds for all $x,y\in D_{_{\MM,Q}}(q; P_k)$. 
\end{lemma}
\noindent{\bf Remark}{
	Note that, if it were $P_k'\subset P_k$ with $k > k'$ (if each $P_{k+1}$ were constructed by adding one point to $P_k$), then  Lemma~\ref{lem:cover} would follow automatically from Lemma~\ref{lem:lambda} and Corollary~\ref{cor:diam}. Since, in general, this not the case, the following proof is needed. 
}
\\
\begin{proof}
	It suffices to show that every Voronoi region $D_{_{\MM,Q}}(q; P_k)$, for sufficiently large $k$, is contained in a neighborhood $V_\lambda(v_q)$ of the type described in Lemma~\ref{lem:lambda}, for some $v_q\in\MM$. 

	Clearly, by Lemma~\ref{lem:lambda}, the set $C = \{V_\lambda(x) : x\in\MM\}$ is an open cover of $\MM$. 
	Since $\MM$ is compact, $C$ admits a finite subcover $C'$. By the Lebesgue number lemma, there is $\delta > 0$ such that every set in $\MM$ of diameter less than $\delta$ is contained in some open set of $C'$. 
	
	Now let $k' = \lceil  (\delta / 2\xi )^{-d}  \rceil$. By Corollary~\ref{cor:diam}, every Voronoi region $D_{_{\MM,Q}}(q; P_k)$, with $q\in P_k$, $k\ge k'$, has diameter less than $\delta$, and is therefore contained in some set of $C'$. 
	Since Equation~\ref{eq:lambda} holds inside every set of $C'$ then, in particular, it holds inside $D_{_{\MM,Q}}(q; P_k)$.
\end{proof}

\noindent We now have all the tools needed to prove: 
\\

\noindent{\bf Lemma~\ref{kfae}}\emph{
	If $\mathcal{F}_k$ is the class of sets of $k$ $d$-dimensional affine spaces, and $\EE^*_{\me,k}$ is the minimizer of Equation~\ref{EErho} over $\mathcal{F}_k$, then there is a constant $C$ that depends on $d$ only, such that
	\begin{equation*}
		\ds{\lim_{k\rightarrow\infty} \EE^*_{\me,k} \cdot k^{4/d} } \le C  \cdot \left(\kappa_{\MM}\right)^{4/d}   
	\end{equation*}
	where $\kappa_\MM := \mu_{|\text{II}|}(\MM)$ is the total root curvature of $\MM$. 
	The constant $C$ grows as $C\sim \left( d / (2\pi e)\right)^2$ with $d\rightarrow\infty$. 
}
\begin{proof}

Pick $\alpha > 0$ and $\lambda > 0$. 
Given $P_k$ minimizing Equation~\ref{eqOQ4}, if $F_k$ is the union of tangent spaces at each $p\in P_k$, by Lemmas~\ref{lem:lambda} and~\ref{lem:cover}, it is
\begin{equation}
\begin{split}
	\EE_\me(F_k) &= \ds{  \int_\MM d\mu_{_\text{I}}(x) \de(x) \min_{p\in P_k} d^2_{_\H}(x, T_p\MM) } \\
				&\le (1+\lambda) \ds{  \int_\MM d\mu_{_\text{I}}(x) \de(x) \min_{p\in P_k} d^4_{_Q}(x, p) } \\
				&= (1+\lambda) \ds{    \int_\MM d\mu_{_Q}(x) \frac{\de(x)}{ \omega_Q(x) } \min_{p\in P_k} d^4_{_Q}(x, p) }  \\
				&\overset{\text{Thm.~\ref{ThGruber}, r=4}}{\le}  (1+\lambda) C   \left\{ \ds{  \int_\MM d\mu_{_Q}(x) \left[  \frac{\de(x)}{ \omega_Q(x) }   \right]^{d/(d+4)} } \right\}^{(d+4)/d} \cdot k^{-4/d}
\end{split}
\end{equation}
where the last line follows from the fact that $P_k$ has been chosen to minimize Equation~\ref{eqOQ4}, and 
	where, in order to apply Theorem~\ref{ThGruber}, we use the fact that $\de$ is absolutely continuous in $\MM$. 

By the definition of $\omega_Q$, it follows that
\begin{equation}
\begin{split}
	\left\{   \ds{  \int_\MM d\mu_{_Q}(x) \left[  \frac{\de(x)}{ \omega_Q(x) }   \right]^{d/(d+4)} } \right\}^{(d+4)/d} &= \left\{ \ds{  \int_\MM d\mu_{_\text{I}}(x) \omega_Q(x)^{4/(d+4)} \de(x)^{d/(d+4)} }  \right\}^{(d+4)/d}    \\
				&\le  
					\left\{ \ds{  \int_\MM d\mu_{_\text{I}}(x)  \omega_{_{Q}}(x) } \right\}^{4/d}					
\end{split}
\end{equation}
where the last line follows from H\"older's inequality ($\|fg\|_1 \le \|f\|_p \|g\|_q$ with $p=(d+4)/d > 1$, and $q=(d+4)/4$.) 

Finally, by the definition of $Q$ and $\alpha'$, it is
\begin{equation}
\begin{split}
	 \ds{  \int_\MM d\mu_{_\text{I}}(x)  \omega_{_{Q}}(x) }   &\le  \ds{   \int_\MM d\mu_{_\text{I}}(x) \alpha + \int_\MM d\mu_{|\text{II}|}(x) } =      \alpha \mathcal{V}_\MM + \kappa_\MM
\end{split}
\end{equation}
where $\mathcal{V}_\MM$ is the total volume of $\MM$, and $\kappa_\MM:=\mu_{|\text{II}|}(\MM)$ is the total root curvature of $\MM$. 
Therefore
\begin{equation}\label{eq:perf}
	\EE_\me(F_k) \le   (1+\lambda) C \left\{    \alpha \mathcal{V}_\MM + \kappa_{_\MM}  \right\}^{4/d}  \cdot k^{-4/d}
\end{equation}

Since $\alpha > 0$ and $\lambda > 0$ are arbitrary, 
Lemma~\ref{kfae} follows. 

Finally, we discuss an important technicality in the proof that we hadn't mentioned before in the interest of clarity of exposition. 
Because we are taking absolutely values in its definition, $Q$ is not necessarily of class $\mathcal{C}^1$, even if $\text{II}$ is. 
Therefore, we may not apply Theorem~\ref{ThGruber} directly.
We may, however, use Weierstrass' approximation theorem (see for example~\cite{dieudonne2008foundations} p.\ 133), to obtain a smooth $\epsilon$-approximation to $Q$, which can be enforced to be positive definite by relating the choice of $\epsilon$ to that of $\alpha$, and with $\epsilon\rightarrow 0$ as $\alpha\rightarrow 0$. Since the $\epsilon$-approximation $Q$ only affects the final performance (Equation~\ref{eq:perf}) by at most a constant times $\epsilon$, then the fact that $\alpha$ is arbitrarily small (and thus so is $\epsilon$) implies the lemma.

\end{proof}

\subsection*{Approximation Error: General Case}

Assume given a $d$-manifold $\MM$ with metric in $\mathcal{C}^3$ embedded in a separable Hilbert space $\H$. 
Consider the definition in Section~\ref{sec:results} of 
the second fundamental form $\text{II}$ and its operator norm $|\text{II}|$. 

We begin extending the results of Lemma~\ref{lem:lambda} to the general case, where the manifold is embedded in a possibly infinite-dimensional ambient space. In this case, the orthogonal complement $(T_x\MM)^\perp$ to the tangent space at $x\in\MM$ may be infinite-dimensional (although, by the separability of $\H$, it has a countable basis.)

For each $x\in\MM$, consider the largest $x$-centered ball $B_x(\varepsilon)$ for which there is a smooth one-to-one Monge patch $m_x:B_x(\varepsilon_x)\subset T_x\MM\rightarrow\MM$. 
Since $\MM$ is smooth, and $\text{II}$ bounded, by the inverse function theorem it holds $\varepsilon_x > 0$. 
Because $\text{II}\in\mathcal{C}^1$, we can always choose $\varepsilon_x$ to be continuous in $\MM$, and thus by the compactness of $\MM$ there is a minimum $0 < \varepsilon$ such that $0 < \varepsilon  \le \varepsilon_x$ with $x\in\MM$. 
Let $N_x(\delta)$ denote the geodesic neighborhood around $x\in\MM$ of radius $\delta$. 
We begin by proving the following technical Lemma.

\begin{lemma}\label{lem:monge}
	For every $q\in\MM$, there is $\delta_q$ such that, for all $x,y\in N_q(\delta_q)$, 
	it is $x \in m_y(B_y(\varepsilon))$ ($x$ is in the Monge patch of $y$.)
\end{lemma}
\begin{proof}
The Monge function $m_y:B_y(\epsilon)\rightarrow\MM$ is such that $r\in B_y(\epsilon)$ implies $m_y(r) - (y + r) \in (T_y\MM)^\perp$ (with the appropriate identification of vectors in $\H$ and in $(T_y\MM)^\perp$), and therefore for all $r\in B_y(\epsilon)$ it holds
	\[ d_{\text{I}}(y,m_y(r)) \ge \| m_y(r) - y\|_\H = \|m_y(r) - (y+r) + (y+r) - y\|_\H = \|m_y(r) - (y+r)\|_\H + \|r\|_\H \ge \|r\|_\H \]
Therefore 
	$N_y(\varepsilon)\subset m_y(B_y(\varepsilon))$. 

For each $q\in\MM$, the geodesic ball $N_q(\varepsilon/2)$ is such that, by the triangle inequality, for all $x,y\in N_q(\varepsilon/2)$ 
	it is $d_{\text{I}}(x,y) \le  \varepsilon$. 
Therefore $x\in N_y(\varepsilon)\subset m_y(B_y(\varepsilon))$.
\end{proof}

\begin{lemma}\label{lem:lambdag}
	For all $\lambda>0$ and $q\in\MM$, there is a neighborhood $V\ni q$ such that, for all $x,y\in V$ 
		it is
		\begin{equation}\label{eq:lambdag}
			d_{\H}^2(x,T_y\MM) \le (1+\lambda) d_{\text{\emph{I}}}^4(x,y) |\text{\emph{II}}_x|^2 
		\end{equation}
\end{lemma}
\begin{proof}
	Let $V$ be a geodesic neighborhood of radius smaller than $\varepsilon$, so that Lemma~\ref{lem:monge} holds. 
	Define the extension ${\text{II}}^*_x(r) = \text{II}^*_x(r^t + r^\perp) := \text{II}_x(r^t)$ of the second fundamental form to $\H$,  where $r^t\in T_x\MM$ and $r^\perp\in (T_x\MM)^\perp$ is the unique decomposition of $r\in\H$ into tangent and orthogonal components. 

	By Lemma~\ref{lem:monge}, given $x,y\in V$, $x$ is in the (one-to-one) Monge patch $m_y$ of $y$. 
	Let $x'\in T_y\MM$ be the unique point such that $m_y(x')=x$, and let $r:=(x'-y) / \|x'-y\|_\H$. 
	Since the domain of $m_y$ is convex, the curve $\gamma_{y,r}:[0,\|x'-y\|_\H]\rightarrow\MM$ given by
		\[ \gamma_{y,r}(t) = y + t r + m_y(t r) =  y + t r + \frac 1 2 t^2 \text{II}_y(r) + o(t^2) \]
	is well-defined, where the last equality follows from the smoothness of $\text{II}$. Clearly, $\gamma_{y,r}(\|x'-y\|_\H)=x$. 

	For $0\le t \le \|x'-y\|_\H$ the length of $\gamma_{y,r}([0,t])$ is
	\begin{equation}\begin{split}
		L(\gamma_{y,r}([0,t])) = \displaystyle{\int_0^t d\tau \| \dot{\gamma_{y,r}}(\tau)\|_{_\H} } 
			= \displaystyle{\int_0^t d\tau  \left(\|r\|_\H + O(t)\right) } = t \cdot (1 + o(1))
	\end{split}\end{equation}
	(where $o(1)\rightarrow 0$ as $t\rightarrow 0$.)
	This establishes the closeness of distances in $T_y\MM$ to geodesic distance on $\MM$. 
	In particular, for any $\alpha>0$, $y\in\MM$, there is a sufficiently small geodesic neighborhood $N\ni y$ such that, for $x\in N$, it holds
		\[ \|x'-y\|_\H \le \|x-y\|_\H \le d_{\text{I}}(x,y)  \le (1+\lambda)\|x'-y\|_\H \]

	By the smoothness of $\text{II}$, for $y\in\MM$ and $x\in N_y(\delta_y)$, with $0 < \delta_y < \varepsilon$, it is 
		\begin{equation*}\begin{split}
			 d^2_\H(x,T_y\MM) &= d^2_\H(\gamma_{y,r}(\|x'-y\|_\H), T_y\MM) = \| \frac 1 2 \text{II}_y(r)\|x'-y\|_\H^2 + o(\|x'-y\|_\H^2)\|^2 \\
			 				&=  \| \frac 1 2 \text{II}^*_y(x-y) + o(\delta_y^2)\|^2 
		\end{split}\end{equation*}
	and therefore for any $\alpha > 0$, there is a sufficiently small $0 < \delta_{y,\alpha} < \varepsilon$ such that, 
	given any $x\in N_y(\delta_{y,\alpha})$, it is 
		\begin{equation}\label{lem:lambdag_cond1}
			  d^2_\H(x,T_y\MM) \le (1+\alpha) \| \frac 1 2 \text{II}^*_y(x-y) \|^2 
		\end{equation}
	By the smoothness of $\text{II}$, and the same argument as in Lemma~\ref{lem:monge}, there is a continuous choice of $0 < \delta_{y,\alpha}$, 
	and therefore a minimum value $0 < \delta_\alpha \le \delta_{y,\alpha}$, for $y\in\MM$. 
		
	Similarly, by the smoothness of $\text{II}^*$, for any $\alpha>0$ and $y\in\MM$, there is a sufficiently small $\beta_{y,\alpha}>0$ such that, 
	for all $x\in N_y(\beta_{y,\alpha})$, it holds
	\begin{equation}\label{lem:lambdag_cond2}
		 \|\frac 1 2 \text{II}^*_y(y-x)\|^2 \le (1+\alpha) \|\frac 1 2 \text{II}^*_x(y-x)\|^2
	\end{equation}
	By the argument of Lemma~\ref{lem:monge}, there is a continuous choice of $0 < \beta_{y,\alpha}$, 
	and therefore a minimum value $0 < \beta_\alpha \le \beta_{y,\alpha}$, for $y\in\MM$. 

	Finally, let $\alpha=\lambda/4$, and restrict $0<\lambda<1$ (larger $\lambda$ are simply less restrictive.) 
	For each $q\in\MM$, let $V = N_q(\min\{\delta_{\alpha}, \beta_{\alpha}\} / 2)\ni q$ be a sufficiently small geodesic neighborhood such that, 
	for all $x,y\in V$, Eqs.~\ref{lem:lambdag_cond1} and~\ref{lem:lambdag_cond2} hold. 
	
	Since $\alpha=\lambda/4 < 1/4$, it is clearly $(1+\alpha)^2 \le (1+\lambda)$, and therefore 
	\begin{equation}\begin{split}
	 d^2_\H(x,T_y\MM) &\le (1+\alpha) \|\frac 1 2 \text{II}^*_y(y-x)\|^2 \le (1+\alpha)^2 \|\frac 1 2 \text{II}^*_x(y-x)\|^2 \\
						  &\le (1+\lambda) \frac 1 4 \|y-x\|^4 |\text{II}_x|^2 \le (1+\lambda) \frac 1 4  d_\text{I}^4(x,y) |\text{II}_x|^2 
	\end{split}\end{equation}
	where the second-to-last inequality follows from the definition of $|\text{II}|$. 
%
%
%
\end{proof}

Note that the same argument as that of Lemma~\ref{lem:cover} can be used here, with the goal of making sure that, 
	for sufficiently large $k$, every Voronoi region of each $p\in P_k$ in the approximation satisfies Equation~\ref{eq:lambdag}.
We may now finish the proof by using a similar argument to that of the codimension-one case.

Let $\lambda >0$. 
Consider a discrete set $P_k\subset\MM$ of size $k$ that minimizes 
\begin{equation}\label{eqOQ4g}
	g(P_k) =  \displaystyle{ \int_{_\MM} d\mu_{_\text{I}}(x) \frac 1 4 \de(x) |\text{II}_x|^2 \min_{p\in P_k} d^4_{_{_\text{I}}}(x,p)  }
\end{equation}
Note once again that the distance and measure in Equation~\ref{eqOQ4g} match and therefore, since $\de(x)|\text{II}_x|^2/4$ is continuous,  we can apply Theorem~\ref{ThGruber} (with $r=4$) in this case. 

Let $F_k := \cup_{q\in P_k} T_q\MM$. By Lemma~\ref{lem:lambdag} and Lemma~\ref{lem:cover}, adapted to this case, 
there is $k'$ such that for all $k\ge k'$ it is
\begin{equation}
\begin{split}
	\EE_\me(F_k) &= \ds{  \int_\MM d\mu_{_\text{I}}(x) \frac 1 4 \de(x) \min_{p\in P_k} d^2_{_\H}(x, T_p\MM) } \\
				&\le (1+\lambda) \ds{  \int_\MM d\mu_{_\text{I}}(x) \frac 1 4  \de(x) |\text{II}_x|^2 \min_{p\in P_k} d^4_{_\text{I}}(x, p) } \\
				&\overset{\text{Thm.~\ref{ThGruber}}, r=4}{\le}  (1+\lambda) C  
				 \left\{ \ds{  \int_\MM d\mu_{_\text{I}}(x) \left[   \frac 1 4 \de(x) |\text{II}_x|^2  \right]^{d/(d+4)} } \right\}^{(d+4)/d} \cdot k^{-4/d}
\end{split}
\end{equation}
where the last line follows from the fact that $P_k$ has been chosen to minimize Equation~\ref{eqOQ4g}.

Finally, by H\"older's inequality, it is
\begin{equation*}\begin{split} 
	\left\{\ds{  \int_\MM d\mu_{_\text{I}}(x) \left[   \frac 1 4 \de(x) |\text{II}_x|^2  \right]^{d/(d+4)} } \right\}^{(d+4)/d} &\le 
		\left\{ \ds{  \int_\MM d\mu_{_\text{I}}(x)\de(x) } \right\} \left\{ \ds{  \int_\MM d\mu_{_\text{I}}(x)  \left(\frac 1 4  |\text{II}_x\|^2 \right)^{d/4} } \right\}^{4/d}\\ 
		&= \| \frac 1 4  |\text{II}|^2 \|_{d/4}
\end{split}\end{equation*}
and thus
\[ 
	\EE_\me(F_k) \le (1+\lambda) C \cdot \left(\kappa_{\MM} / k\right)^{4/d}
\]
where the total curvature $\kappa_{\MM} :=  \ds{  \int_\MM d\mu_{_\text{I}}(x) \frac 1 4  |\text{II}_x|^{d/2} }$ 
is the geometric invariant of the manifold (aside from the dimension) that controls the constant in the bound.  

Since $\alpha > 0$ and $\lambda > 0$ are arbitrary, 
Lemma~\ref{kfaepp} follows.

\subsection*{Proofs of Theorems~\ref{thkf} and~\ref{thkfg}}


We use the results discussed in Section~\ref{sec:derivation} to obtain the proof of Theorem~\ref{thkf} as follows. 
The proof of Theorem~\ref{thkfg} follows from the derivation in Section~\ref{sec:derivation}, as well as the argument below, with $\kappa^1_{_\MM}$ substituted by $\kappa_{_\MM}$, and is omitted in the interest of brevity. 

\begin{proof}
By Lemmas~\ref{kfse} and~\ref{kfae}, 
with probability $1-\delta$, it is
\begin{equation}
\begin{split}
	\EE_\me(\Fnk) &\le 2 n^{-1/2} \left( k\sqrt{2\pi d} + \sqrt{\frac{1}{2} \ln 1/\delta} \right) + C (\kappa^1_{_\MM} / k)^{4/d} \\
				&\le 2 n^{-1/2}  k\sqrt{2\pi d} \cdot \sqrt{\frac{1}{2} \ln 1/\delta} + C (\kappa^1_{_\MM} / k)^{4/d} \\
				&= 2\left(8\pi d\right)^{2/(d+4)} C^{d/(d+4)} \cdot n^{-2/(d+4)}  \cdot \sqrt{\frac{1}{2}\ln 1/\delta} \cdot   \left(\kappa^1_{_\MM}\right)^{4/(d+4)}
\end{split}
\end{equation}
where the last line follows from choosing  $k$ to balance the two summands of the second line, as:
\[
	k_n =  n^{\frac{d}{2(d+4)}} \cdot \left(  \frac{C}{2\sqrt{2\pi d}}  \right)^{d/(d+4)}  \cdot \left(\kappa^1_{_\MM}\right)^{4 / (d+4)}
\]

\end{proof}

\end{document}